\documentclass[final]{siamltex}
\usepackage[final]{graphicx}
\usepackage[dvips,all,color]{xy}
\xyoption{arc}
\usepackage{subfig}
\usepackage{amsmath,amssymb}
\usepackage[svgnames,x11names]{xcolor}
\usepackage{dsfont}
\usepackage[linesnumbered,ruled,vlined,longend]{algorithm2e}
\usepackage{mathrsfs}
\usepackage{stmaryrd}
\usepackage{ifsym}
\usepackage{here}
\usepackage{psfrag}
\usepackage{mathtools}
\usepackage[normalem]{ulem}
\newcommand{\forloop}[5][1] { \setcounter{#2}{#3} \ifthenelse{#4} { #5 \addtocounter{#2}{#1} \forloop[#1]{#2}{\value{#2}}{#4}{#5} }  Else { } }
\DeclareMathAlphabet{\mathpzc}{OT1}{pzc}{m}{it}

\newtheorem{cor}{Corollary}
\newtheorem{lem}{Lemma}
\newtheorem{prop}{Proposition}
\newtheorem{defn}{Definition}

\newtheorem{rem}{Remark}
\newtheorem{notn}{Notation}

\newcommand{\Gnet}{\mathds{G}_{\textrm{\sffamily\textsc{N}}}}
\newcommand{\cur}{\textrm{}}
\newcommand{\Sn}{\textrm{\sffamily\textsc{Tgt}}}

\newcommand{\Crd}{\textrm{\sffamily\textsc{Card}}}

\newcommand{\eg}{\geqq_{\textbf{\texttt{(Elementwise)}}}}
\newcommand{\Q}{\mathscr{P}}
\newcommand{\Pitilde}{\widetilde{\Pi}}
\newcommand{\mudble}[2]{\widehat{\nu}^\infty_{\left(#1,#2\right)}}
\newcommand{\nuinf}[2][\theta]{\widehat{\nu}^\infty_{#1} #2}

\newcommand{\ones}{\mathbf{e}}
\newcommand{\myar}{\ar@[|(2.5)]}
\newcommand{\myarT}{\ar@[|(3.5)]}
\newcommand{\myarL}{\ar@[|(1.5)]}
\newcommand{\upd}{\mathds{U}}
\newcommand{\Tc}{\mathds{T}_c}
\newcommand{\cgather}[2][0pt]{\begingroup\setlength\abovedisplayskip{#1}\setlength\belowdisplayskip{#1}\begin{gather} #2 \end{gather}\endgroup}
\newcommand{\cgathers}[2][0pt]{\begingroup\setlength\abovedisplayskip{#1}\setlength\belowdisplayskip{#1}\begin{gather*} #2 \end{gather*}\endgroup}
\newcommand{\calign}[2][0pt]{\begingroup\setlength\abovedisplayskip{#1}\setlength\belowdisplayskip{#1}\begin{align} #2 \end{align}\endgroup}
 
\newcommand{\Mblue}{\color{MidnightBlue}}
\newcommand{\Nblue}{\color{DodgerBlue}}
\newcommand{\Dgreen}{\color{DarkGreen}}
\newcommand{\BRed}{\color{DarkRed}}

\newcommand{\red}{\color{Red2}}
\newcommand{\black}{\color{black}}
\title{
 Distributed Self-Organization Of Swarms  To Find Globally $\epsilon$-Optimal Routes To Locally Sensed Targets
}
\author{ 
Ishanu~Chattopadhyay
\thanks{
Corresponding Author, email: ixc128@psu.edu, Mechanical Engineering, The Pennsylvania State University,  USA (Supported in part by the U.S.  Army Research Office (W911NF-07-1-0376) and  the Office of Naval Research (N00014-09-1-0688)) 
} 
} 

\begin{document}
\maketitle
\allowdisplaybreaks{

\begin{abstract}The problem of near-optimal distributed path planning to locally sensed targets is investigated in the context of  large swarms.
The proposed algorithm uses only information that can be locally queried, and 
rigorous theoretical results on convergence, robustness, scalability are established, and effect of system parameters such as the agent-level communication radius and agent velocities on global performance is analyzed.
The fundamental philosophy of the proposed approach is to percolate local information across the swarm, enabling agents to indirectly access the global context. A gradient emerges, reflecting the performance of agents, computed in a distributed manner via local information exchange between neighboring agents. It is shown that to follow near-optimal routes to a target which can be only sensed locally, and whose location is not known a priori, the agents need to simply move towards its ``best'' neighbor, where the notion of ``best'' is obtained by computing the state-specific language measure of an underlying probabilistic finite state automata. The theoretical results are validated in high-fidelity simulation experiments, with excess of $10^4$ agents.
\end{abstract}

\begin{AMS}\end{AMS}

\pagestyle{myheadings}
\thispagestyle{plain} 
\markboth{I. Chattopadhyay}{ Distributed Self-Organization Of Swarms  To Find  $\epsilon$-Optimal Routes }
\section{Introduction \& Motivation}
Path planning in a co-operative  environment is a problem of great interest in multi-agent robotics.   Recent developments in micro machining and MEMs have opened up the possibility of engineering extremely small and cheap robotic platforms in large numbers. Limited in size, on-board computational resources and power, such robots nevertheless can potentially exploit co-operation  to accomplish complex tasks \cite{DG89,Gage1992,Holl99,Unsal1994} including surveillance,
reconnaissance, path finding and collaborative payload conveyance.
However, coordinating such engineered swarms unveils new challenges not encountered in the operation of one or a few robots~\cite{Gage93,LB92,Payton01}. Coordination schemes  requiring unique identities
for each robot, explicit routing of point-to-point communication between robots, or centralized
representations of the state of an entire swarm are no longer viable. Thus, any approach to  effectively control swarms must be intrinsically scalable, and must only use information that is locally available. The immediate question for the control theorist is whether such algorithms are able to guarantee any level of global performance.
This is precisely the problem that is investigated in this paper, with an affirmative answer; a distributed scalable control algorithm is proposed  that allows very large swarms (simulation results obtained with $10^4$ agents) to self-organize and find near-global-optimal routes to locally known targets.
The proposed algorithm uses only information that can be locally queried, and 
rigorous theoretical results on convergence, robustness, scalability are established, and effect of system parameters such as the agent-level communication radius and agent velocities on global performance is analyzed.

The fundamental philosophy of the proposed approach is to percolate local information across the swarm, enabling agents to indirectly access the global context. A gradient emerges, reflecting the performance of agents, computed in a distributed manner via local information exchange between neighboring agents. It is shown that to follow near-optimal routes to a target which can be only sensed locally, and whose location is not known a priori, the agents need to simply move towards its ``best'' neighbor, where the notion of ``best'' is obtained by computing the state-specific language measure~\cite{CR07} of an underlying probabilistic finite state automata~\cite{CR08}.

Gradient based method in swarm control are not new~\cite{Payton01,STMRK09,NCD08,HW09}.
Majority of reported work following this direction draw
inspiration from swarming phenomena observed in nature, where self-organized exploration strategies emerge at the collective level as a result of simple rules
followed by individual agents. To produce the global behavior, individuals interact by using simple  and mostly local
communication protocols. Social insects are a good biological example of
organisms collectively exploring an unknown environment, and they have
 served as a source of direct inspiration for research on self-organized
cooperative robotic exploration and path formation in groups of robots~\cite{Svennebring04,WL99}.
%
The standard engineering approach to analyze desired global patterns and  break
them down into a set of simple rules governing individual agents is seldom applicable to
large populations aspiring to accomplish complex tasks. Nevertheless some progress have been made in this direction~\cite{Couzin2005}. 
\begin{quote}{\itshape 
``We now know that such synchronized group behavior (of flocking birds) is mediated through sensory modalities such as vision, sound, pressure and odor detection. Individuals tend to maintain a personal space by avoiding those too close to themselves; group cohesion results from a longer-range attraction to others; and animals often align their direction of travel with that of nearby neighbors. These responses can account for many of the group structures we see in nature, including insect swarms and the dramatic vortex-like mills formed by some species of fish and bat. By adjusting their motion in response to that of near neighbors, individuals in groups both generate, and are influenced by, their social context — there is no centralized controller.}'' Collective Minds, D. Couzin~\cite{Couzin2007}
\end{quote}
On the other hand, the
Evolutionary Robotics (ER) methodology~\cite{Nolfi00} allows
for an implementation of a top-down approach, where  reinforcement learning via evolutionary optimization
techniques allows assessment of the system’s overall performance, and sequentially improve control laws.
While such heuristic techniques have been shown to yield robust and scalable systems, 
assuring global performance has remained an elusive challenge. The present paper aims to 
fill this gap by proposing a simple control approach with provable guarantees on global performance. The key difference with the reported gradient based techniques lies in the 
formal model that is developed, and the associated theoretical results that show that the algorithm  achieves near-global optimality.
%

To the best of the author's knowledge, such an approach has not been previously investigated, primarily due to the complexity spike encountered in deriving optimal solutions in a decentralized environment.
Recent investigations~\cite{Bn00,BGIZ02} into the solution complexity of  decentralized Markov decision processes  have shown that the problem is exceptionally hard  even for two agents;  illustrating a fundamental
divide between centralized and decentralized control of MDP. In contrast
to the centralized approach, the decentralized case provably does not admit
polynomial-time algorithms. Furthermore, assuming $\textrm{EXP} = \textrm{NEXP}$, the problems require super-exponential time to solve in the worst case.
Furthermore, since distributed systems with access to only local information can be mapped to partially observable MDPs, it follows from \cite{LGM01} that such problems are non-approximable, negating the possibility of obtaining optimal solutions to approximate representations.

Such negative results do not preclude the possibility of obtaining {\itshape near-optimal} solutions efficiently, when the set of models considered is a strictly smaller subset of general MDPs.  This is precisely what we achieve in this paper; casting the path planning problem as a performance maximization problem for an underlying probabilistic finite state automata (PFSA). In spite of similar Markovian assumptions, the PFSA model is distinct from the general MDPS (See Section~\ref{subsecDP}), and admits decentralized manipulation, such that the control policy, on convergence, is within an $\epsilon$ bound of the global optimal. Furthermore, one can freely choose the error bound $\epsilon$ (and make it as small as one wishes), with the caveat that the convergence time 
increases (with no finite upper bound) with decreasing $\epsilon$.
%

The present work is also distinct from Potential Field-based  methodologies (PFM) widely studied in the centralized single-or-few robot scenarios~\cite{KGZ07,SHK08}. Early PFM implementations had substantial shortcomings~\cite{BK91-1} suffering from trap situations, instability in narrow passages $etc$.
Some of these shortcomings have been addressed recently, leading to globally convergent potential planners~\cite{Volpe1990,Connolly1997,Vascak2007,KK93,S93,AKH08,WM03}. These approaches are computationally hard for single-or-few robots, and 
thus not applicable in the current context. Some variations of the latter approaches have attempted to reduce the complexity by
 combining  search algorithms and
potential fields \cite{Shar1993,Barraquand1992, Barraquand1991}, virtual obstacle methods
method \cite{Park2003,Li2000}, sub-goal methods
\cite{Bell2004,Weir2006}, wall-following methods
\cite{Borenstein1989,Yun1997,Park2003,Mabrouk2008w}  etc. Nevertheless, since
heuristic strategies only based on local environment
information are usually applied, many of these methods cannot
guarantee  convergence in general.

The rest of the  paper is organized in   six sections. Section~\ref{sec2} briefly summarizes the theory of quantitative measures of probabilistic regular languages, and the pertinent approaches to centralized performance maximization of PFSA. Section~\ref{sec3} develops the PFSA model for a swarm, and  Section~\ref{sec4} presents the  theoretical development for decentralized PFSA optimization, thus solving the problem of computing $\epsilon$-optimal routes in a static or frozen swarm. Section~\ref{secMob} 
extends the results to a dynamic swarm, where route optimization and positional updates are carried out simultaneously. 
Section~\ref{sec6} validates the theoretical development with high fidelity simulation results. The paper concludes in Section~\ref{sec7} with recommendations for future work.
\section{Background: Language Measure Theory}\label{sec2}
This section summarizes the concept of signed real measure of probabilistic regular
languages, and its application in performance optimization of 
probabilistic finite state automata (PFSA)~\cite{CR07}.
A string over an alphabet ($i.e.$ a non-empty finite set) $\Sigma$ is a finite-length sequence of symbols from $\Sigma$~\cite{HMU01}. 
The Kleene closure of $\Sigma$, denoted by $\Sigma^*$, is the set of all finite-length strings of symbols including the null string $\epsilon$.
$xy$ is the  concatenation of strings $x$ and $y$, and the null string $\epsilon$ is the identity element of the concatenative monoid.
\begin{defn}[PFSA]\label{defPFSA}
A PFSA $G$ over an alphabet $\Sigma$ is a sextuple $(Q,\Sigma,\delta,\widetilde{\Pi},\chi,$ $\mathscr{C})$, where $Q$ is a set of states, $\delta:Q\times\Sigma^\star \rightarrow Q$ is
the (possibly partial) transition map; $\widetilde{\Pi}: Q\times \Sigma \rightarrow [0,1]$ is an output mapping or  the probability morph function that specifies the state-specific symbol generation probabilities, satisfying
$\forall q_i \in Q, \sigma \in \Sigma, \widetilde{\Pi}(q_i,\sigma) \geqq 0$, and $ \sum_{\sigma\in\Sigma}\widetilde{\Pi}(q_i,\sigma)=1$, the state characteristic function $\chi:Q \rightarrow [-1, 1]$ 
assigns a signed real weight to each state reflecting the immediate pay-off from visiting that state, and $\mathscr{C}$ is the set of controllable transitions that can be disabled (See Definition~\ref{defcontapp}) by an imposed control policy.
\end{defn}
\begin{defn}[Control Philosophy]\label{defcontapp}
If $\delta(q_i,\sigma) = q_k$, then the \textit{disabling}
of  $\sigma$ at 
$q_i$ prevents the state transition from $q_i$ to $q_k$.  Thus, disabling a transition $\sigma$ at a  state $q$ replaces 
the original transition with a  self-loop  with identical occurrence probability, $i.e.$ we now have $\delta(q_i,\sigma) = q_i$. 
Transitions that can be so disabled are 
\textit{controllable}, and belong to the set $\mathscr{C}$.
\end{defn}\vspace{0pt}
\begin{defn}\label{Lgen}
The language $L(q_i)$ generated by a PFSA $G$ initialized at the
state $q_i\in Q$ is defined as:
$    L(q_i) = \{s \in \Sigma^* \ | \ \delta(q_i, s) \in Q \}$
Similarly, for every $q_j\in Q$,  $L(q_i, q_j)$ denotes the set of all
strings that, starting from the state $q_i$, terminate at the
state $q_j$, i.e.,
$L(q_i,q_j) = \{ s \in \Sigma^* \ | \ \delta(q_i, s) = q_j \in Q \}$
\end{defn}
\begin{defn}[State Transition Matrix]\label{pifn}The state transition probability matrix $\Pi \in [0,1]^{\Crd(Q) \times \Crd(Q)}$,
for a given PFSA is defined as:
$\forall q_i, q_j \in Q, \Pi_{ij} =
     \sum_{\sigma\in\Sigma \ \mathrm{s.t.} \  \delta(q_i,\sigma)=q_j } \widetilde{\Pi}(\sigma, q_i)$
Note that $\Pi$ is a square non-negative stochastic matrix~\cite{BR97}, where $\Pi_{ij}$ is the probability of transitioning from  $q_i$ to $q_j$.
\end{defn}\vspace{0pt}
\begin{notn}
We use matrix notations interchangeably for the morph function $\widetilde{\Pi}$. In particular,
$\widetilde{\Pi}_{ij} = \widetilde{\Pi}(q_i,\sigma_j)$ with $ q_i \in Q, \sigma_j \in \Sigma$.
Note that $\widetilde{\Pi} \in [0,1]^{\Crd(Q) \times \Crd(\Sigma)}$ is not necessarily square, but each row sums up to unity.
\end{notn}
A signed real measure~\cite{R88} $\nu^i:{2^{L(q_i)}} \rightarrow
\mathbb{R}\equiv(-\infty,+\infty)$ is constructed on the
$\sigma$-algebra $2^{L(q_i)}$~\cite{CR07}, implying that every singleton string set $\{ s \in L(q_i) \} $ is a measurable set. 
\begin{defn}[Language Measure]\label{measurefn}Let $\omega \in L(q_i, q_j)\subseteq 2^{L(q_i)}$. The signed
real measure $\nu^i_\theta$ of every singleton string set $ \{ \omega
\} $ is defined as:
$\nu^i_\theta(\{\omega \})\triangleq\theta (1-\theta)^{\vert \omega \vert}\widetilde{\Pi}(q_i,\omega)\chi(q_j)$.
For every choice of the parameter $\theta \in (0,1)$, the signed real measure of a sublanguage $L(q_i,q_j) \subseteq
L(q_i)$ is defined as:
$ \nu^i_\theta(L(q_i, q_j)) \triangleq \sum_{\omega\in
L(q_i, q_j)} \theta (1-\theta)^{\vert \omega \vert}\widetilde{\Pi}( q_i,\omega)\chi_j
$. The measure of $L(q_i)$, is defined as 
$\nu^i_\theta(L(q_i)) \triangleq \sum_{q_j \in Q}
\nu^i_\theta(L_{i,j})$.
\end{defn}
\begin{notn}
 For a given PFSA, we interpret the set of measures $\nu^i_\theta(L(q_i))$ as a real-valued vector of length $\Crd(Q)$ and  denote $\nu^i_\theta(L(q_i))$ as $\nu_\theta \vert_i$.
\end{notn}
The language measure can be expressed vectorially as (where the inverse  exists for $\theta \in (0,1]$~\cite{CR07}):
\begin{gather}\label{eqmesd}
 \nu_\theta = \theta \big [ \mathbb{I} - (1-\theta)\Pi \big ]^{-1} \chi 
\end{gather}

In the limit of $\theta \rightarrow 0^+$, the language measure of singleton strings can be interpreted to be product of the conditional generation probability of the string, and the 
characteristic weight on the terminating state. Hence, smaller the characteristic, or smaller the probability of generating the string, smaller is its measure. Thus, if the
characteristic values are chosen to represent the control specification, with more positive weights given to more desirable states, then the measure represents how \textit{good} the particular string is with respect to the given specification, and the given model. The limiting language measure $\nu_0\vert_i = \lim_{\theta \rightarrow 0^+}\theta \big [ \mathbb{I} - (1-\theta)\Pi \big ]^{-1} \chi \big \vert_i$
sums up the limiting measures of each string starting from $q_i$, and thus captures how \textit{good} $q_i$ is, based on not only its own characteristic, but on how \textit{good} are the  strings  generated in  future from $q_i$. It is thus a quantification of the impact of $q_i$, in a probabilistic sense, on future dynamical evolution~\cite{CR07}.

\begin{defn}[Supervisor]A supervisor is a control policy  disabling a specific subset of the set $\mathscr{C}$ of
controllable transitions. Hence there is a bijection
between the set of all possible supervision  policies and the
power set $2^{\mathscr{C}}$. \end{defn}

Language measure  allows quantitative comparison of
different supervision policies.
\begin{defn}[Optimal Supervision Problem]\label{pdef}
Given $G=(Q,\Sigma,\delta,\widetilde{\Pi},\chi,\mathscr{C})$, compute a supervisor disabling 
$\mathscr{D}^\star \subseteq \mathscr{C}$, s.t. $
\nu^{\star}_0 \eg \nu^{\dag}_0 \ \
\forall \mathscr{D}^{\dag} \subseteq \mathscr{C} $ where
$\nu^{\star}_0$, $\nu^{\dag}_0$ are
the limiting measure vectors of  supervised plants $G^{\star}$, $G^{\dag}$ under $\mathscr{D}^{\star}$, $\mathscr{D}^\dag$ respectively.
\end{defn}

The solution to the optimal supervision problem  is obtained
in \cite{CR07} by designing an optimal policy using $\nu_\theta$ with $\theta \in (0,1)$. To ensure
that the computed optimal policy coincides with the one for
$\theta \rightarrow 0^+$, the authors choose a \textit{small non-zero}
value for $\theta$ in each iteration step of the design
algorithm. 
To address numerical issues, algorithms reported in \cite{CR07}
computes how small a $\theta$ is actually sufficient to ensure that the optimal solution computed with this value of $\theta$ coincides with the optimal policies for any smaller value, $i.e.$,
computes the critical lower bound $\theta_\star$. (This is closely related to the notion of Blackwell optimality; See Section~\ref{subsecDP})
Moreover the solution obtained is stationary,  efficiently computable, and can be shown to be the unique  maximally permissive policy among ones with maximal performance.
Language-measure-theoretic optimization is \textit{not a search} (and has several key advantages over Dynamic Programming based approaches. See Section~\ref{subsecDP} for details); it is an iterative sequence of combinatorial manipulations, that 
monotonically improves the measures, leading to  element-wise maximization of $\nu_\theta$ (See \cite{CR07}).
It is shown in \cite{CR07} that
 $\lim_{\theta \rightarrow 0^+}  \theta \big [ \mathbb{I} - (1-\theta)\Pi \big ]^{-1} \chi = \Q \chi$,
where the $i^{th}$ row of $\Q$ (denoted as $\wp^i$) is the stationary probability vector for the PFSA initialized at state $q_i$. 
In other words, $\Q$ is the Cesaro limit of the stochastic matrix $\Pi$, satisfying
$ \Q = \lim_{k\rightarrow\infty}\frac{1}{k}\sum_{j=0}^{k-1}\Pi^j$~\cite{BR97}.
\begin{prop}[See \cite{CR07}]\label{prop1}
Since the optimization maximizes the language measure element-wise for $\theta \rightarrow 0^+$, it follows that 
for the optimally supervised plant, the standard inner product $\langle \wp^i,\chi\rangle$ is maximized, irrespective of the starting state $q_i \in Q$.
\end{prop}
\begin{notn}
The optimal $\theta$-dependent measure for a PFSA is denoted as $\nu^\star_\theta$
and the limiting measure  as $\nu^\star$.
\end{notn}
\begin{algorithm}[!ht]
  \SetKwData{Left}{left}
  \SetKwData{This}{this}
  \SetKwData{Up}{up}
  \SetKwFunction{Union}{Union}
  \SetKwFunction{FindCompress}{FindCompress}
  \SetKwInOut{Input}{input}
  \SetKwInOut{Output}{output}
  \SetKw{Tr}{true}
   \SetKw{Tf}{false}
  \caption{Computation of Optimal Supervisor}\label{Algorithm02}
\Input{$\mathbf{P}, \ \boldsymbol{\chi}, \ \mathscr{C}$}
\Output{Optimal set of disabled transitions
$\mathscr{D}^{\star}$} \Begin{ Set
$\mathscr{D}^{[0]}=\emptyset$ \tcc*[r]{Initial disabling set}
Set $\widetilde{\Pi}^{[0]}=\widetilde{\Pi} $ \tcc*[r]{Initial
event prob. matrix}
Set $\theta^{[0]}_{\star} = 0.99$, Set $k \ = \ 1$ , Set $
\texttt{Terminate} = $ \Tf \;
 \While{($ \texttt{Terminate}$ == \Tf)}{
 Compute $\theta_{\star}^{[k]}$\tcc*[r]{Algorithm~\ref{Algorithm01}}
 Set $\widetilde{\Pi}^{[k]} = \frac{1-\theta_{\star}^{[k]}}{1-\theta_{\star}^{[k-1]}}\widetilde{\Pi}^{[k-1]}$\;
 Compute $\boldsymbol{\nu}^{[k]}$ \;
 \For{$j = 1$ \textbf{to} $n$}{
\For{$i = 1$ \textbf{to} $n$}{
 Disable  all controllable
  $q_i \xrightarrow[]{\sigma} q_j$ s.t.
$\boldsymbol{\nu}^{[k]}_j < \boldsymbol{\nu}^{[k]}_i $ \;
 Enable all controllable   $q_i \xrightarrow[]{\sigma} q_j$
 s.t.
$\boldsymbol{\nu}^{[k]}_j \geqq \boldsymbol{\nu}^{[k]}_i $ \;
} } Collect all disabled transitions in $\mathscr{D}^{[k]}$\;
\eIf{$\mathscr{D}^{[k]} ==
\mathscr{D}^{[k-1]}$}{$\texttt{Terminate} =$ \Tr \;}{$k \ = \
k \ + \ 1$ \;}
 }
 $\mathscr{D}^{\star} \ = \ \mathscr{D}^{[k]}$
 \tcc*[r]{Optimal disabling set}
  }\vspace{0pt}
\end{algorithm}
\begin{algorithm}[!ht]
  \SetKwInOut{Input}{input}
  \SetKwInOut{Output}{output}
  \SetKw{Tr}{true}
   \SetKw{Tf}{false}
  \caption{Computation of the Critical Lower Bound $\theta_{\star}$ }\label{Algorithm01}
\Input{$\mathbf{P}, \ \boldsymbol{\chi}$}
\Output{$\theta_{\star}$} \Begin{ Set $\theta_{\star} =
1$, Set $\theta_{curr} = 0$\; Compute $\Q$ , $M_0$ ,
 $M_1 $,
  $M_2 $\;
 \For{$j = 1$ \textbf{to} $n$}{
\For{$i = 1$ \textbf{to} $n$}{ \eIf{$\left (\Q
\boldsymbol{\chi} \right )_i - \left (\Q \boldsymbol{\chi}
\right )_j \neq 0$}{$\theta_{curr} = \frac{1}{8M_2}\big \vert
\left (\Q \boldsymbol{\chi} \right )_i - \left (\Q
\boldsymbol{\chi} \right )_j \big \vert $}{
\For {$r=0$ \textbf{to} $n$}{ \eIf{$\left
(M_0\boldsymbol{\chi} \right )_i \neq \left
(M_0\boldsymbol{\chi} \right )_j$}{\textbf{Break}\;}{\If
{$\left ( M_0 M_1^r \boldsymbol{\chi} \right )_i \neq \left (
M_0 M_1^r \boldsymbol{\chi} \right )_j$}{\textbf{Break}\;} } }
\eIf{$r==0$}{$\theta_{curr} = \frac{\vert \left \{ (M_0 -\Q
)\boldsymbol{\chi} \right \}_i - \left \{ (M_0 -\Q
)\boldsymbol{\chi} \right \}_j \vert}{8M_2} $\;} {\eIf { $r >
0$
 \textbf{AND} $ r \leq n $ }{ $ \theta_{curr} = \frac{\vert
\left (M_0 M_1\boldsymbol{\chi} \right )_i - \left (M_0 M_1
\boldsymbol{\chi} \right )_j \vert }{2^{r+3}M_2} $ \;}{ $
\theta_{curr} = 1 $ \;}}}
$\theta_{\star}$ = $\mathrm{min} (
\theta_{\star} , \theta_{curr} ) $ \;
 } } }
\end{algorithm}

For  completeness, the key algorithms are included as Algorithms~\ref{Algorithm02} and \ref{Algorithm01}.
\subsection{Relation Of The Centralized Approach To Dynamic Programming}\label{subsecDP}
In spite of underlying Markovian assumptions, the PFSA model is distinct from the standard formalism of (finite state) Controlled Markov Decision Processes (CMDP)~\cite{FS02,Bertsekas1978,Bertsekas1987}.  In  the latter, control actions are not probabilistic; the associated control function maps  states to unique actions in a  deterministic manner, and  the control problem is to decide which of the available control actions should be executed in each state. On the other hand, in the PFSA formalism, control is exerted by selectively disabling controllable probabilistic state transitions, and is thus a probabilistic generalization of supervisory control theory~\cite{RW87}. {\itshape Note that while in the MDP framework, one specifies which control action to take at a given state, in the PFSA formalism one specifies which of the available control actions are \underline{not} allowed at the current state, and that any of the remaining can be executed in accordance to their generation probabilities.}
Denoting the set of controllable transitions at state $q_k$ as $\Sigma_C$, and  $\phi:Q \rightarrow \Sigma_C$ as the control policy  mapping the current state $q_k\in Q$ to the controllable move $\phi(q_k) \in \Sigma_C$ (and assuming that the control action is to dictate the agent to execute a specific controllable move and is not supervisory in nature), one can formulate an analogous optimization problem that admits solution within the DP framework. The transition probabilities for a stationary  policy $\phi$ is given by:
\begin{gather}
\forall q_i,q_j \in  Q, \sigma_r \in \Sigma_C,
     Prob(q_j \big \vert q_i, \phi(q_i) = \sigma_r) = \sum_{\mathclap{ \sigma \textrm{ s.t. }\left(\sigma \in \Sigma\setminus\Sigma_C \cup \{\sigma_r\} \right )\bigwedge \left (\delta(q_i,\sigma) = q_j\right)}}\tilde{\pi}(q_i, \sigma) 
\end{gather}
Immediate rewards, in DP terminology, 
can be  related to the state characteristic $\chi$:
\begin{gather}
     g(q_i,\phi(q_i)) = \chi(q_i)
\end{gather}
We note that the problem at hand must be solved over an infinite horizon, since the total number of transitions ($i.e.$ the path length) is not bounded.
Identifying $(1-\theta)$ as the discount factor,  the cost-to-go (to be maximized) for the infinite horizon discounted cost (DC) problem is given by:
\begin{gather}
     J(q_0,\phi)=\lim_{N\rightarrow \infty} \mathbf{E}\left ( \sum_{i=0}^N (1-\theta)^i g\big(q_i,\phi(q_i)\big) \right )
\end{gather}
and for the infinite horizon Average Cost per stage (AC):
\begin{gather}
     J(q_0)=\lim_{N\rightarrow \infty} \frac{1}{N} \mathbf{E}\left ( \sum_{i=0}^N  g\big(q_i,\phi(q_i)\big) \right )
\end{gather}
AC is more appropriate, since there is no reason to "discount"  events in future. In the PFSA formalism, we solve the analogous discounted problem at sufficiently small $\theta$, and  guarantee that the solution is simultaneously average cost optimal, primarily  due to the following identity~\cite{CR07}:
\begin{gather}
     \lim_{\theta \rightarrow 0^+} \theta\sum_{k=0}^\infty (1-\theta)^k\Pi^k\chi = \lim_{N\rightarrow \infty} \frac{1}{N} \sum_{j=0}^{N-1} \Pi^j \chi = \Q\chi
\end{gather}
where $\Q$ is the Cesaro limit of the stochastic matrix $\Pi$. Thus, the proposed technique can solve the problem by maximizing  
\begin{gather}
\nu_\theta=\lim_{\theta \rightarrow 0^+} \theta\sum_{k=0}^\infty (1-\theta)^k\Pi^k\chi=\lim_{\theta \rightarrow 0^+} \theta \big [ \mathbb{I} - (1-\theta)\Pi \big ]^{-1} \chi
\end{gather}
$i.e.$, the language measure, and achieve maximization of $\Q\chi$  (guaranteeing the probability of reaching the goal is maximized, while simultaneously minimizing collision probability). In any case, the formulated  DP problem can  be solved using standard solution methodologies such as Value Iteration (VI) or Policy Iteration (PI)~\cite{Put90,GS75}. We note that for VI, we need to search for the control action that maximizes the value update over all possible control actions, in each iteration. 
On the other hand, PI involves two steps in each iteration: (1) policy evaluation,
which is very similar to the measure computation step in each iteration for the language-measure-theoretic technique and (2) policy improvement, which involves 
searching for a improved action over possible control actions for each state (which involves at least one product between a $\Crd(Q)\times \Crd(Q)$ matrix of transition probabilities and the current cost vector of length $\Crd(Q)$ per state). The disabling/enabling of controllable transitions  is significantly simpler compared to the search steps that both VI and PI require, and the 
improvement in complexity (for PI which is closer to the proposed algorithm in the centralized case, since the latter proceeds via computing a sequence of monotonically improving policies) is by at least an asymptotic factor of $O(\Crd(Q)^2)$ per iteration (for dense matrices, and $O(\Crd(Q))$ per iteration in the sparse case), which is significant for large problems. 
\begin{rem}
     The number of iterations is not expected to be comparable for the PFSA framework, VI and PI techniques; simulations indicate the measure-theoretic approach converges faster, and detailed investigations in this direction is a topic of future work.
\end{rem}

Another key advantage of the PFSA-based solution methodology  is guaranteed Blackwell optimality~\cite{ABEFG93,FS02}. It is well recognized that the average cost criterion is underselective, namely the finite time behavior is completely ignored. The condition of Blackwell optimality attempts to correct this by demanding the computed controller be optimal for a continuous range of discount factors in the interval $(d_0,1)$, $i.e.$ for $\theta \in (0, 1-d_0)$, where $0< d_0 < 1$. Since the PFSA-based approach maximizes the language measure for some $\theta = \theta_{min}$, such that the optimal policy is guaranteed to be identical for all values of $\theta$ in the range $[\theta_{min}, 0]$, the solution satisfies the Blackwell condition. It is possible to obtain such  Blackwell optimal policies within the DP framework as well, but the approach(es) are significantly more involved (See \cite{FS02}, Chapter 8).
The ability to adapt $\theta$ at each iteration (See Algorithm~\ref{Algorithm01}) leads to  a novel adaptive discounting scheme in the technique proposed, which solves the infinite horizon problem efficiently while using a non-zero $\theta \geqq \theta_{min}$ at all iteration steps.
\section{The Swarm Model}\label{sec3}
 \begin{figure}[t]
\centering 
\includegraphics[width=4.75in]{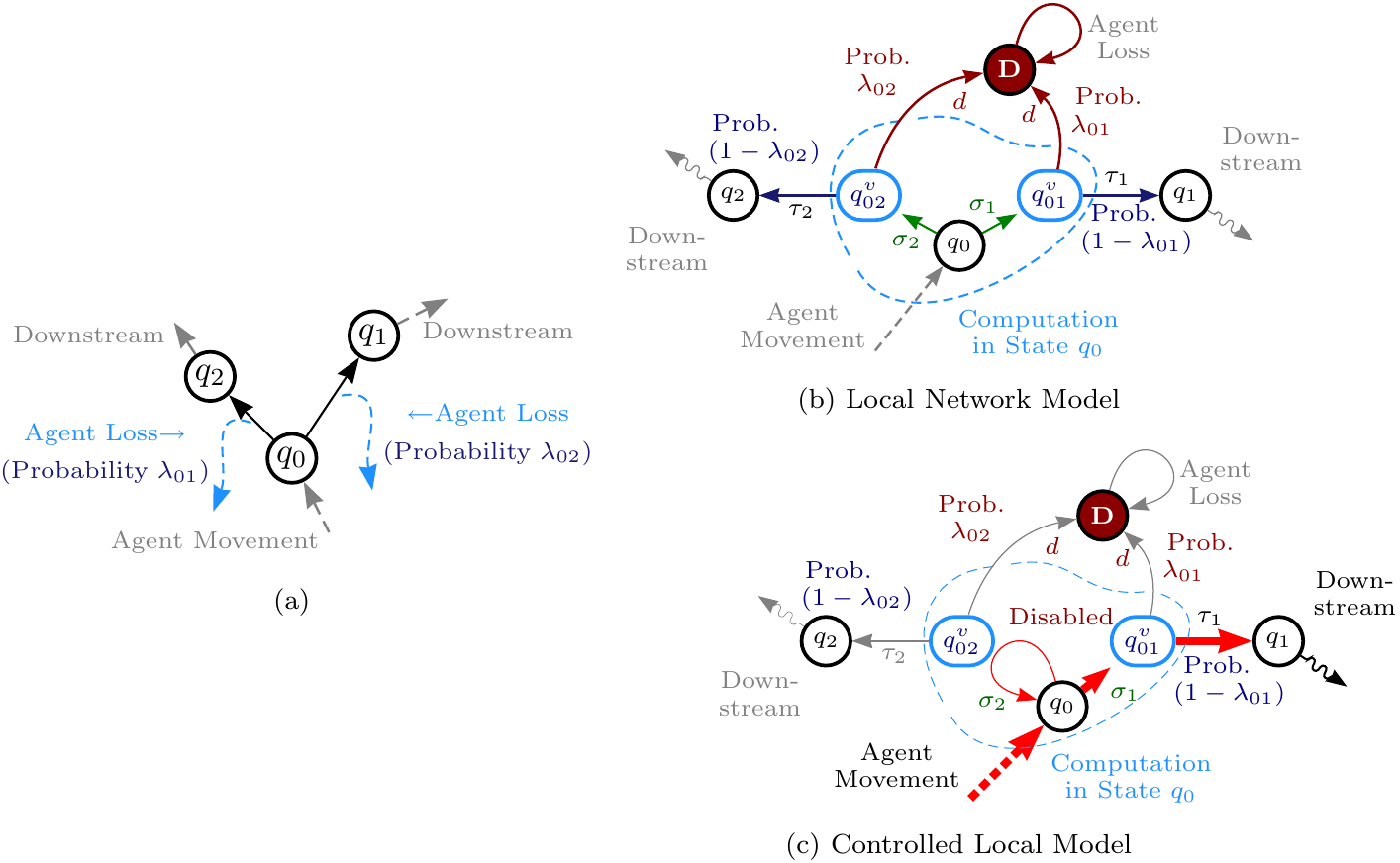}
\caption{Agent-centric local decision-making with non-zero failure probability}\label{figsimT2}
\end{figure}
We consider  ad-hoc mobile network of communicating agents endowed with limited computational resources. For simplicity of exposition, 
we develop the theoretical results under the assumption of a single target, or goal. This is not a serious  restriction and can be easily relaxed.
The location and identity of the target is not known a priori to the individual agents, only ones which are within the communication radius of the target can sense its presence.
The communication radii are assumed to be constant throughout.
 Inter-agent communication links  are assumed to be perfect, which
again can be generalized easily, within our framework. We  assume agents can efficiently gather the following information:
\begin{enumerate}
 \item \textit{(Set of Neighboring agents:)} Number and unique id. of agents to which it can successfully send data via a 1-hop direct line-of-sight link. The communication radius $R_c$ is assumed to be identical for each agent. The set of neighbors for each agent varies 
with time as the swarm evolves.
\item \textit{(Local Navigation Properties:)} Navigation is assumed to occur by moving  towards  a chosen neighbor with constant velocity, the magnitude of which is assumed to be identical for each agent. 
In general, there a non-zero probability of agent failure in the course of execution this maneuver, which is assumed to be either known or learnable by the agents. However the explicit learning of these 
local costs is not addressed in this paper.
\end{enumerate}
The local network model, along with the decision-making philosophy, is illustrated in Figure~\ref{figsimT2}.
We will talk about a \textit{frozen} swarm, which denotes a particular spatial configuration of the agents assumed to be fixed in time. Unless explicitly mentioned, the agents 
are assumed to be updating their positions in continuous time (moving with constant velocity), while changing their headings in discrete time as dictated via on-board decision-making based on locally available information, with the objective of reaching the target with minimum end-to-end probability of agent failure.
We further assume that the failure probabilities  mentioned above are functions of the agent locations, and possibly vary in a smooth non-increasing manner with increasing inter-agent distances. However no time-dependence is assumed, $i.e.$ the failure probabilities remain  constant for a frozen swarm, and change (due to the positional updates) for a mobile one.  The agent velocities are assumed to be significantly slower compared to the 
time required for the convergence of the optimization algorithm for each frozen configuration. The implications of the last assumption will be discussed in the sequel.
First we formalize  a failure-prone ad-hoc network of frozen communicating agents as a probabilistic finite state automata.
\begin{defn}[Neighbor Map For A Frozen Swarm]\label{defneighbor}
If $Q$ is the set of all agents in the network, then the neighbor map  $\mathcal{N}:Q \rightarrow 2^Q$ specifies, for each agent $q_i \in Q$, the set of agents $\mathcal{N}(q_i) \subset Q$ (excluding $q_i$) to which $q_i$ can communicate via a single hop direct link.
\end{defn}
\begin{defn}[Failure Probability]\label{deffailureprob}
The  failure probability $\lambda_{ij} \in [0,1]$ is defined to be the probability of unrecoverable loss of agent $q_i$ in the course of moving towards  agent $q_j$.
\end{defn}

Thus, $\lambda_{ij}$ reflects local or immediate navigation costs, and estimated risks and therefore  varies with the positional coordinates of the agents $q_i$ and $q_j$.
These quantities are not constrained to be symmetric in general, $i.e.$,   $\lambda_{ij} \neq \lambda_{ji}$. We assume the agent-based estimation of these ratios to converge fast enough, in the scenario where such parameters are learned on-line. Since we are more concerned with  decision optimization in this paper, we ignore the parameter estimation problem of learning the failure probabilities, which is at least intuitively justified by the  existence of separated policies in large classes of similar problems. 

We  visualize the local network around a agent $q_0$ in a manner illustrated in Figure~\ref{figsimT2}(a) (shown for two neighbors $q_1$ and $q_2$). In particular,  agent $q_0$ attempting to move  towards the current position of $q_1$ experiences  a failure probability $\lambda_{01}$, while the moving towards $q_2$ has a failure probability $\lambda_{02}$. To correctly represent this information, we require the notion of \textit{virtual states} ($q^v_{01},q^v_{02}$ in Figure~\ref{figsimT2}(b)). 
\begin{rem}[Necessity Of Virtual States]
 The virtual states are required to model the physical situation within the PFSA framework, in which transitions do not emerge from other transitions. As illustrated in Figure~\ref{figsimT2}(a), the 
failure events do actually occur in the course of the attempted maneuver; hence necessitating the notion of the virtual states.
\end{rem}
\begin{defn}[Virtual State]\label{defvirtualagent}
Given a agent $q_i$, and a  neighbor $q_j \in \mathcal{N}(q_i)$ with 
a specified failure probability $\lambda_{ij}$, any attempted move towards $q_j$  is assumed to be first routed to a virtual state $q^v_{ij}$, upon which there is either an automatic ($i.e.$ uncontrollable) forwarding to  $q_j$ with probability $1-\lambda_{ij}$, or a failure with probability $\lambda_{ij}$. 
The set of all virtual states in a network of $Q$ agents is denoted by $Q^v$ in the sequel.\end{defn}
Hence, the total number of virtual states is given by:
\begin{gather}
 \Crd(Q^v) = \sum_{i:q_i \in Q} \mathcal{N}(q_i)
\end{gather}
And the cardinality of the set of virtual states satisfies:
\begin{gather}\label{eqbnd1}
 0 \leqq \Crd(Q^v) \leqq \Crd(Q)^2 - \Crd(Q)
\end{gather}
\begin{figure}[t]
\centering
 \includegraphics[width=3.25in]{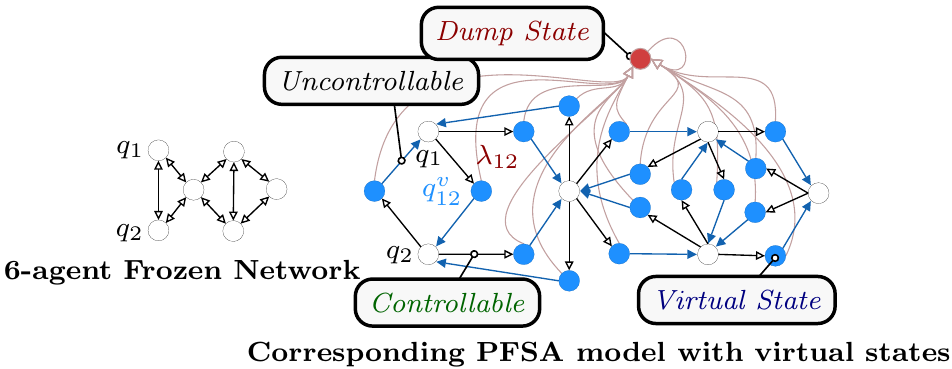}
\caption{6-agent network \&  23 state PFSA   (16 virtual states, 6 agents, 1 dump state)}\label{fignetpfsa}
\end{figure}
We assume that there is a static agent at the  target or the goal, which we denote as $q_\Sn$. The local communication with this agent-at-target can be visualized as the process of sensing the target by the mobile agents.
We are  ready to  model  an ad-hoc communicating network of frozen agents as a PFSA, whose states correspond to either agents, the virtual states, or the state reflecting agent failures. 
\begin{defn}[PFSA Model of Frozen Network]\label{defpfsanetwork}
 For a given set of agents $Q$, the function $\mathcal{N}:Q \rightarrow 2^Q$,  the link specific failure probabilities $\lambda_{ij}$ for any agent $q_i$ and a neighbor $q_j \in \mathcal{N}(q_i)$, and a specified target $q_\Sn \in Q$, the PFSA $\Gnet = (Q^N,\Sigma,\delta,\widetilde{\Pi},\chi,\mathscr{C})$ is defined to be a model of the network, where (denoting $\Crd(\mathcal{N}(q_i))=m$):
\begin{subequations}
\begin{align}
&\textrm{\scshape \sffamily \footnotesize $\circ$ States:}
  &&Q^N = Q \bigcup Q^v \bigcup \big \{q_D\big \}
\intertext{where $Q^v$ is the set of virtual states, and $q_D$ is a dump state which models loss of agent due to failure. 
}
&\textrm{\scshape \sffamily \footnotesize $\circ$ Alphabet:}
 &&\Sigma = \bigcup_{i:q_i \in Q} \left ( \bigcup_{j: q_j \in \mathcal{N}(q_i)} \sigma_{ij} \right )\bigcup \big \{\sigma_D\big \} 
\intertext{where $\sigma_{ij}$ denotes  navigation (attempted or actual)
$\sigma_D$ denotes agent failure.}
&\txt{\scshape \sffamily \footnotesize $\circ$ Transition Map: } \ 
&&
\delta(q,\sigma) = \left \{ \begin{array}{cl}
       q^v_{ij} & \textrm{if } q = q_i , \sigma = \sigma_{ij} \\
q_j & \textrm{if } q = q^v_{ij},\sigma = \sigma_{ij}\\
q_D & \textrm{if } q = q^v_{ij},\sigma = \sigma_D \\
q_D & \textrm{if } q = q_D,\sigma = \sigma_D \\
- & \textrm{undefined otherwise}
                            \end{array}
\right.\\
&
\txt{\scshape \sffamily \footnotesize $\circ$ Probability  Morph  Matrix:}
&& \widetilde{\Pi}(q,\sigma) = \left \{ \begin{array}{cl}
       \frac{1}{m} & \textrm{if } q = q_i , \sigma = \sigma_{ij} \\
1-\lambda_{ij} & \textrm{if } q = q^v_{ij},\sigma = \sigma_{ij}\\
\lambda_{ij} & \textrm{if } q = q^v_{ij},\sigma = \sigma_D \\
1 & \textrm{if } q = q_D,\sigma = \sigma_D \\
0 & \textrm{otherwise}
                            \end{array}
\right.\\
&\txt{\scshape \sffamily \footnotesize $\circ$ Characteristic   Weights:}
&&\chi_i = \left \{ \begin{array}{cl}
       1 & \textrm{if } q_i = q_\Sn            \\
0 & \textrm{otherwise}
                  \end{array}
\right.\\
&\txt{\scshape \sffamily \footnotesize $\circ$  Controllable  Transitions:}
 && \forall q_i \in Q, q_j \in \mathcal{N}(q_i), q_i \xrightarrow{\sigma_{ij}} q^v_{ij} \in \mathscr{C}
\end{align}
\end{subequations}
\end{defn}
We note  that for a network of $Q$ agents, the PFSA model may have (almost always has, see Figure~\ref{fignetpfsa}) a significantly larger number of states. Using Eq.~\eqref{eqbnd1}:
\begin{gather}
\Crd(Q) + 1 \leqq \Crd(Q^N) \leqq \Crd(Q)^2 + 1
\end{gather}
This state-explosion will not be a problem for the distributed approach developed in the sequel, since we use the complete model $\Gnet$ only for the purpose of deriving theoretical guarantees.
Note, that Definition~\ref{defpfsanetwork} generates a PFSA model which can be optimized in a straightforward manner using the language-measure-theoretic technique described in Section~\ref{sec2} (See \cite{CR07}) for details). This would yield the optimal routing policy in terms of the disabling decisions at each agent that minimize source-to-target failure probabilities (from every agent in the network).   To see this explicitly, note that the measure-theoretic approach elementwise maximizes 
 $\lim_{\theta \rightarrow 0^+}  \theta \big [ \mathbb{I} - (1-\theta)\Pi \big ]^{-1} \chi = \Q \chi$, 
where the $i^{th}$ row of $\Q$ (denoted as $\wp^i$) is the stationary probability vector for the PFSA initialized at state $q_i$ (See Proposition~\ref{prop1}). Since, the dump state has characteristic $-1$, the target has characteristic $1$, and all other agents have characteristic $0$, it follows that this optimization maximizes the quantity $\wp^i_\Sn - \wp^i_{\textrm{\scshape Dump}}$, for every source state or agent $q_i$ in the network. Note that $\wp^i_\Sn, \wp^i_{\textrm{\scshape Dump}}$ are the stationary probabilities of reaching the target and incurring an agent loss to dump respectively, from a given source $q_i$.
Thus, maximizing $\wp^i_\Sn - \wp^i_{\textrm{\scshape Dump}}$ for every $q_i \in Q$ guarantees that the computed routing policy is indeed optimal in the stated sense.
However, the procedure in \cite{CR07}  requires  centralized computations, which is precisely what we wish to avoid. 
The key technical contribution in this paper is to  develop a distributed approach to language-measure-theoretic PFSA optimization. In effect, the theoretical development in the next section allows us to carry out the language-measure-theoretic optimization of a given PFSA, in situations where we do not have access to the complete $\Pi$ matrix, or the $\chi$ vector at any particular agent ($i.e.$ each agent has a limited local view of the network), and are restricted to communicate only with immediate neighbors.
We are interested in not just computing the measure vector in a distributed manner, but 
optimizing the PFSA via selected disabling of controllable transitions (See Section~\ref{sec2}).
This is  accomplished by Algorithm~\ref{AlgorithmOPT}.
\subsection{Control Approach For Mobile Agents}
For the mobile network, $\Gnet(t)$ varies as a function of operation time $t$. For a particular instant $t=t_0$, the globally optimized model $\Gnet^\star(t_0)$ yields the 
local decisions for the agent maneuvers. As stated before, this global optimization can be carried out in a distributed manner, and the agents update their headings towards the neighbor which has the highest measure among all neighbors, provided it is in fact higher than the self-measure. Transition towards 
any neighbor with a better or equal measure compared to self via randomized choice is also acceptable, but we use the former approach for the theoretical development in the sequel. The movement however modifies the PFSA model to $\Gnet(t+\Delta t)$, and a re-optimization is required.
We assume, as stated before, that the agent velocities are slow enough so that they do not interfere with this computation. A crucial point is the time complexity of convergence of the distributed algorithm, which in our case, is small enough to allow this procedure to be carried out efficiently. Also, note that since the complete model is never assembled, the modifications to $\Gnet$ is also
a local affair, $e.g.$ updating the set of neighbors or the failure probabilities, and such local effects are felt by the remote agent via percolated information involving an unavoidable delay, which goes to ensure that the effect of all local changes are not felt simultaneously across the network.
\subsection{Possibility Of Different Local Models}
The local model, as described above and illustrated in Figure~\ref{figsimT2}, assumes that errors are non-recoverable; hence the possibility of transitioning to the dump state, from which no outward transition is defined. Alternatively, we could eliminate the dump state, and simply add the transition as a self-loop, or even redistribute the probability among the remaining transitions defined at a state. It is intuitively clear that the  adopted model avoids errors most aggressively.
\section{Decentralized PFSA Optimization For The Frozen Swarm}\label{sec4}
\setlength{\algomargin}{0em} 
\begin{algorithm}[t]
\fontsize{8}{8}\selectfont
 \SetLine
\linesnotnumbered
\dontprintsemicolon
  \SetKwInOut{Input}{input}
  \SetKwInOut{Output}{output}
 \SetKw{Tr}{true}
   \SetKw{Tf}{false}
  \caption{{  Distributed Update of Agent Measures In Frozen Swarm}}\label{AlgorithmOPT}
\Input{$\Gnet=(Q,\Sigma,\delta,\widetilde{\Pi},\chi,\mathscr{C})$, $\theta$}
\Begin{
Initialize $\forall q_i \in Q, \widehat{\nu}^\cur_\theta \vert_i = 0$\;
\txt{\sffamily \footnotesize \BRed $/* \ $ {Begin Infinite \uline{Asynchronous} Loop} $\ */ $ } 
\While{\Tr}{
\For{each agent $q_i \in Q$}{
\If{ $\mathcal{N}(q_i) \neq \varnothing $}{

$m = \Crd(\mathcal{N}(q_i))$\;
\For {each agent $q_j \in \mathcal{N}(q_i)$}{
\BlankLine
\Nblue
\txt{\sffamily \footnotesize \color{DodgerBlue2} $/* \ $ { (a1) Inter-agent Communication} $\ */ $ \color{black}\\ $\phantom{.}$}%
\vspace{-5pt}
\colorbox{LightCyan}{
  \Mblue Query $\widehat{\nu}^\cur_\theta \vert_j$ \& Failure Prob. $\lambda_{ij}$\;} \color{black}
\BlankLine
\BlankLine
\txt{\sffamily \footnotesize \color{IndianRed4} $/* \ $ { (a2) Control Adaptation} $\ */ $ \color{black}}%
\colorbox{MistyRose}{
\begin{minipage}{2in}
\eIf{$ \widehat{\nu}^\cur_\theta \vert_j < \widehat{\nu}^\cur_\theta \vert_i$}{
$\Pi_{ii} = \Pi_{ii} + \Pi_{i(q^V_{ij})}$\;
$\Pi_{i(q^V_{ij})} = 0$\color{black}\tcc*[r]{\Dgreen  \footnotesize  {\bf \sffamily   Disable} \color{black}}
} {
\If{ $\Pi_{i(q^V_{ij})} == 0$}{
$\Pi_{i(q^V_{ij})} = \frac{1}{m}$\;
$\Pi_{ii} = \Pi_{ii} - \frac{1}{m}$ \color{black}\txt{\sffamily \footnotesize \Dgreen $/* \ $ {\bf \sffamily Enable} $\ */ $ \color{black}}
}
}
\end{minipage}
}
\vspace{2pt}
\BlankLine
\txt{ \sffamily \footnotesize \color{Gold4} $/* \ $ { (a3) Updating Virtual States} $\ */ $ \color{black}\\ $\phantom{.}$} %
\vspace{-5pt}
\Mblue
\colorbox{Beige}{
$\widehat{\nu}^\cur_\theta \vert_{(q_{ij}^V)} = (1-\theta)(1 - \lambda_{ij}) \widehat{\nu}^\cur_\theta \vert_j$}\color{black}
}
}
\txt{ \sffamily \footnotesize \color{Chartreuse4} $/* \ $ { (a4) Updating  Agent} $\ */ $ \color{black}\\ $\phantom{.}$} %
\vspace{-5pt}
\BRed
\colorbox{Honeydew}{
\txt{
$
\widehat{\nu}^\cur_\theta \vert_i = \displaystyle\mspace{-25mu} \sum_{j: q_j \in \mathcal{N}(q_i)} \mspace{-20mu}(1 -\theta) \Pi_{i(q^V_{ij})} \widehat{\nu}^\cur_\theta \vert_{(q_{ij}^V)} $\\ \hspace{80pt}$+ (1-\theta)\Pi_{ii} \widehat{\nu}^\cur_\theta \vert_i + \theta\chi \vert_i
$}
} 
\color{black}
}}
}
\end{algorithm}

 In the sequel, the current measure value, for a given $\theta$, at agent $q_i \in Q$ is denoted as $\widehat{\nu}^\cur_\theta \vert_i$, and the measure of the virtual state $q^v_{ij} \in Q^N$ is denoted as $\widehat{\nu}^\cur_\theta \vert_{(q_{ij}^V)}$. The parenthesized entry $(q_{ij}^V)$ denotes the index  of the virtual state $q^v_{ij}$ in the state set $Q^N$. Similarly, the 
transition probability from $q_i$ to $q^v_{ij}$ is denoted as $\Pi_{i(q_{ij}^V)}$. The subscript entry $i(q_{ij}^V)$ denotes the 
$ik^{th}$ element of $\Pi$, where $k = (q_{ij}^V)$.

Algorithm~\ref{AlgorithmOPT} establishes a distributed, asynchronous  procedure achieving:
\cgather[5pt]{ \forall q_i \in Q, \widehat{\nu}_\theta \vert_i  \xrightarrow[\txt{\footnotesize convergence}]{\txt{\footnotesize global }} \nu^\star_\theta \vert_i}
where $\nu^\star_\theta \vert_i$ is the optimal measure for $q_i \in Q$ that would be obtained by optimizing the PFSA $\Gnet$, for a given $\theta$, in a centralized approach (See Section~\ref{sec2}).
The optimal routing policy can then be obtained by moving towards neighboring agents which 
have a better or equal current measure value. If more than a one such neighbor is available, then one either chooses the local destination agent randomly, in an equiprobable manner; or as we use in this paper, move towards one chosen from the set of neighbors with maximal measure.
In the sequel we show that this forwarding policy  converges to the  globally optimal 
routing policy, that, for a sufficiently small $\theta$, it maximizes probability of reaching the target, while 
simultaneously minimizing the probability of end-to-end failures. Furthermore, choosing randomly between qualifying neighboring agents leads to 
significant congestion resilience. These issues would be elaborated in the sequel (Proposition~\ref{propchar}). 

Algorithm~\ref{AlgorithmOPT} has four distinct parts, marked as (a1), (a2), (a3) and (a4). Part (a1) involves 
inter-agent communication, to enable a particular agent $q_i \in Q$ to ascertain the current measure values of neighboring agents, and the 
failure probabilities $\lambda_{ij}$ on respective links. Recall, that we assume the probabilities $\lambda_{ij}$ to be more or less constant for the frozen swarm; however agents need to  
estimate these values for generalization to the mobile case. Part (a2) is the control adaptation, in which the agents decide, based on local information, the set of allowable destination agents. Part (a3) is the computation of the updated measure values for the virtual states $q^v_{ij}$ where $j: q_j \in \mathcal{N}(q_i)$. Finally, part (a4) updates the
measure of the agent $q_i$ based on the computed current measures of the virtual states. We note that Algorithm~\ref{AlgorithmOPT} only uses information that 
is either available locally, or that which can be queried from neighboring agents.

\begin{prop}[Convergence]\label{propmain}
For a  network $Q$ modeled as  $\Gnet=(Q^N,\Sigma,\delta,\widetilde{\Pi},$ $\chi, \mathscr{C})$, the distributed procedure 
in Algorithm~\ref{AlgorithmOPT} has the following properties:
 \begin{enumerate}
  \item Computed measure values for every agent $q_i \in Q$ are non-negative and bounded above by $1$, $i.e.$,
\begin{gather}
 \forall q_i \in Q^N, \forall t \in [0,\infty), \ \widehat{\nu}^t_\theta \vert_i \in [0, 1]
\end{gather}
\item For constant failure probabilities and constant neighbor map $\mathcal{N}:Q \rightarrow 2^Q$, Algorithm~\ref{AlgorithmOPT} converges in the sense:
\begin{gather}
 \forall q_i \in Q^N, \lim_{t \rightarrow \infty}\widehat{\nu}^t_\theta \vert_i =  \nu^\infty_\theta \vert_i \in [0,1]
\end{gather}
\item Convergent measure values coincide with the optimal values computed by the centralized approach:
\begin{gather}
 \forall q_i \in Q^N, \nu^\infty_\theta \vert_i =  \nu^\star_\theta \vert_i
\end{gather}

 \end{enumerate}

\end{prop}
\begin{proof}
\textit{(Statement 1:)} Non-negativity of the measure values is obvious. 
For establishing the upper bound, we use induction on computation time $t$. 
We note that all the measure values $\widehat{\nu}^t_\theta \vert_i$ are initialized to $0$ at time $t=0$.
The first agent to change its measure will be the target, which is updated at some time $t=t_0$:
\begin{gather}\label{eqt0}
 \widehat{\nu}^{t_0}_\theta \vert_{(q_\Sn)} = 0 + \theta\chi_{(q_\Sn)} = \theta
\end{gather}
where the first term is zero since all agents still have measure zero and the target  characteristic $\chi_{(q_\Sn)} = 1$.
Thus,  there exists a non-trivial time instant $t_0$, at which:
\begin{align}
\textrm{\sffamily {\small (Induction Basis)} }  \forall q_i \in Q^N,  \widehat{\nu}^t_\theta \vert_i \leqq 1
\end{align}
Next we assume for  time $t=t'$, we have
\begin{align*}
\textrm{\sffamily {\small (Induction Hypothesis)} }  \forall q_i \in Q^N, \forall \tau \leqq t', \widehat{\nu}^\tau_\theta \vert_i \leqq 1
\end{align*}
We consider the next updates for physical agents and virtual states separately, and denote the time instant for the next updates as $t'_+$. Note, that 
$t'_+$ actually may be different for different agents (asynchronous operation).\\
{\itshape (Virtual States)}
For any virtual state $q_i  = q^v_{kj} \in Q^N$, where $q_k,q_j \in Q$, we have:
\begin{gather}
 \widehat{\nu}^{t'_+}_\theta \vert_i = (1-\lambda_{ij})(1-\theta)\widehat{\nu}^{t'}_\theta \vert_j \leqq 1
\end{gather}
{\itshape (Physical Agents)}
For any  $q_i   \in Q$, where  set of enabled neighbors
$E_n = \big \{q_j \in \mathcal{N}(q_i) \ \textrm{s.t. } \widehat{\nu}^{t'_+}_\theta \vert_{(q_{ij}^V)} \geqq \widehat{\nu}^{t'}_\theta \vert_i\big \}$:
\begin{multline*}
 \widehat{\nu}^{\tau+}_\theta \vert_i =\frac{1}{\Crd(\mathcal{N}(q_i))}\bigg (\sum_{j:q_j \in E_n} (1-\theta)^2(1-\lambda_{ij})\widehat{\nu}^{t'}_\theta \vert_{(q_{ij}^V)} + \sum_{j: j \in \mathcal{N}(q_i) \setminus E_n}(1-\theta)\widehat{\nu}^{t'}_\theta \vert_i \bigg ) \\\leqq
\frac{1}{\Crd(\mathcal{N}(q_i))}\bigg (\sum_{j:q_j \in E_n} 1 + \sum_{j: j \in \mathcal{N}(q_i) \setminus E_n}1 \bigg )
\leqq  1
\end{multline*}
which establishes Statement 1.\\
\textit{(Statement 2:)}
We claim that for each agent $q_i \in Q^N$, the sequence of measures $\widehat{\nu}^t_\theta\vert_i$ forms a monotonically non-decreasing sequence as a function of the computation time $t$. Again, we use induction on computation time. 
Considering the time instant $t_0$ (See Eqn.~\eqref{eqt0}), we note that we have an instant up to which all measure values have indeed changed in a non-decreasing fashion, since the measure 
of $q_\Sn$ increased to $\theta$, while  other agents are still at $0$; which establishes the  basis.
For our  hypothesis, we assume that there exists  some time instant $t' > t_0$, such that all measure values  have undergone non-decreasing updates up to $t'$.
We consider the physical agent $q_i \in Q$ which is the first one to update next, say at the instant $t'_+ > t'$. Referring to Algorithm~\ref{AlgorithmOPT}, this update occurs by first updating the set of virtual states $\{q_{ij}^v : q_j \in \mathcal{N}(q_i)\}$. Since  virtual states update as:
\begin{gather}
 \widehat{\nu}^{t'_+}_\theta \vert_{(q^v_{ij}} = (1-\theta)(1-\lambda_{ij})\widehat{\nu}^{t'}_\theta \vert_j
\end{gather}
it follows from the induction hypothesis that 
\begin{gather}
 \widehat{\nu}^{t'_+}_\theta \vert_{(q^v_{ij})} \geqq \widehat{\nu}^{t'}_\theta \vert_{(q^v_{ij})}
\end{gather}
 If the connectivity ($i.e.$ the forwarding decisions) for the physical agent $q_i$  remains unchanged for the instants $t'$ and $t'_+$, and since 
the measures of any neighboring agent has not decreased (by induction hypothesis), then:
\begin{gather}
 \widehat{\nu}^{t'_+}_\theta \vert_i \geqq \widehat{\nu}^{t'}_\theta \vert_i
\end{gather}
If, on the other hand, the set of disabled transitions for $q_i$ changes ($e.g.$ for some $q_j \in \mathcal{N}(q_i)$, $q_i \xrightarrow{\sigma_{ij}}q_{ij}^v$ was disabled at $t'$ and is enabled at $t'_+$, or vice verse), the measure of agent $q_i$ is increased by the additive factor $\frac{(1-\theta)}{\Crd(\mathcal{N}(q_i))}\bigg\vert\widehat{\nu}^{t'}_\theta \vert_i - \widehat{\nu}^{t'}_\theta \vert_{(q^v_{ij})} \bigg \vert  $, which completes the inductive process and establishes our claim that the measure values form a non-decreasing sequence for each agent as a function of the computation time. Since, a non-decreasing bounded sequence in a complete space must converge to a unique limit~\cite{R88},  the convergence: 
\begin{gather}
 \forall q_i \in Q^N, \lim_{t \rightarrow \infty}\widehat{\nu}^t_\theta \vert_i =  \nu^\infty_\theta \vert_i \in [0,1]
\end{gather}
  follows from the 
existence of the upper bound established in Statement 1. This establishes Statement 2.\\
\textit{(Statement 3:)}
From the update equations in Algorithm~\ref{AlgorithmOPT}, we note that the limiting measure values satisfy:
\begin{align}
\widehat{\nu}_\theta^{\infty} \big \vert_i &= (1-\theta)\sum_{j \in \mathcal{N}(i)} \Pi_{ij} \widehat{\nu}_\theta^{\infty} \big \vert_j + \theta \chi \vert_i \notag\\
\Rightarrow \widehat{\nu}_\theta^{\infty} &= \theta \big [ \mathbb{I} - (1-\theta)\Pi \big ]^{-1}\chi
\end{align}
which implies that measure values does indeed converge to the measure vector computed in a centralized fashion (See Eq.~\eqref{eqmesd}).
Noting that any further disabling (or re-enabling) would not increase the measure values computed by Algorithm~\ref{AlgorithmOPT}, we conclude that this must be 
the optimal disabling set that would be obtained by the centralized language-measure theoretic optimization of  PFSA $\Gnet$ (Section~\ref{sec2}). This completes the proof.
\end{proof}

\begin{prop}[Initialization Independence]\label{propinitindependence}
For a  network $Q$ modeled as a PFSA $\Gnet=(Q^N,\Sigma,\delta,\widetilde{\Pi},\chi, \mathscr{C})$, convergence of Algorithm~\ref{AlgorithmOPT} is independent of the initialization of the measure values, $i.e.$, if $\widehat{\nu}_{\theta,\alpha}^t$  denotes  the measure vector at time $t$ with arbitrary initialization  $\alpha \in [0,1]^{\Crd(Q^N)}$, then:
\begin{gather}\label{eqprop1}
 \lim_{t\rightarrow \infty}\widehat{\nu}_{\theta,\alpha}^t = \lim_{t\rightarrow \infty}\widehat{\nu}_{\theta}^t
\end{gather}
where $\widehat{\nu}_{\theta,\alpha}^0= \alpha$ and $ \widehat{\nu}_{\theta}^0= [0 \cdots 0]^T$.
\end{prop}
\begin{proof} The measure update equations in Algorithm~\ref{AlgorithmOPT} dictate that the measure values will have a positive contribution from $\alpha$. Denoting the contribution of $\alpha$ to the measure of agent $q_i \in Q$ at time $t$ as $\mathcal{C}^t_\alpha(q_i)$, we note that the measure can be written as $\widehat{\nu}_{\theta,\alpha}^t = \mathcal{C}^t_\alpha(q_i) + f_i^t$, where $f_i^t$ is independent of $\alpha$.
Furthermore, the linearity of the updates imply that $\mathcal{C}^t_\alpha(q_i)$ can be used to formulate an inductive argument as follows. We use $k_\star^t \in \mathbb{N}\cup \{0\}$  to denote the minimum number of updates  that every agent in the network has encountered up to  time instant $t \in [0,\infty)$. 
We claim that:
\begin{gather}\label{eqclaim1}
 \forall q_i \in Q, \forall t \in [0,\infty), \mathcal{C}^t_\alpha(q_i) \leqq (1-\theta)^{k_\star^t} \vert \vert \alpha \vert \vert_1
\end{gather}
To establish this claim, we use induction on $k_\star^t$.
For the  basis, we note that there exists a time instant $t_0$, such that $\forall \tau \leqq t_0, k_\star^\tau =0$, implying that 
\begin{gather*}
 \forall \tau \leqq t_0, \mathcal{C}^\tau_\alpha(q_i) = \alpha_i \leqq (1-\theta)^0 \sum_{q_j \in Q} \alpha_j = (1-\theta)^{k_\star^\tau} \vert \vert \alpha \vert \vert_1
\end{gather*}
We assume that if at some  $t_k$,  $k_\star^{t_k} = k \in \mathbb{N}$, then:
\begin{align*}
\textrm{\sffamily {\small (Induction Hypothesis)} }  \forall q_i \in Q, \mathcal{C}^{t_k}_\alpha(q_i) \leqq (1-\theta)^k \vert \vert \alpha \vert \vert_1
\end{align*}
Next let $q_i$ be an arbitrary physical agent, and consider the first update of $q_i$ at $t_k^+ > t_k$:
\begin{align*}
\widehat{\nu}^{t_k^+}_\theta \vert_i = & \mspace{0mu} \sum_{j: q_j \in \mathcal{N}(q_i)} \mspace{-20mu}(1 -\theta) \Pi_{i(q^v_{ij})} \widehat{\nu}^{t_k}_\theta \vert_{(q_{ij}^v)} + (1-\theta)\Pi_{ii} \widehat{\nu}^{t_k}_\theta \vert_i + \theta\chi_i \\
\Rightarrow & \mathcal{C}^{t_k^+}_\alpha (q_i)\leqq  \mspace{0mu}\sum_{j: q_j \in \mathcal{N}(q_i)} \mspace{-20mu}(1 -\theta) 
\Pi_{i(q^v_{ij})} (1-\lambda_{ij})(1-\theta)(1-\theta)^k\vert \vert \alpha \vert \vert_1 \\  & \mspace{150mu}+ (1-\theta)\Pi_{ii}(1-\theta)^k\vert \vert \alpha \vert \vert_1 + \theta\chi_i\\
\Rightarrow &\mathcal{C}^{t_k^+}_\alpha (q_i)\leqq (1-\theta)^{k+1}\vert \vert \alpha \vert \vert_1 
\end{align*}
We note that if $k_\star^{t_{k+1}} = k+1$, then every agent $q_i \in Q$ must have undergone one more update 
since  $t_k$ implying:
\begin{gather}
\forall q_i \in Q, \mathcal{C}^{t_{k+1}}_\alpha (q_i)\leqq (1-\theta)^{k+1}\vert \vert \alpha \vert \vert_1
\end{gather}
which completes the induction proving  Eq.~\eqref{eqclaim1}.
Observing that $\lim_{t\rightarrow\infty} k_\star^t = \infty$, and  $\vert \vert \alpha \vert \vert_1 < \infty$, we conclude:
\begin{gather}
 \forall q_i \in Q, \lim_{t\rightarrow\infty} \mathcal{C}^t_\alpha(q_i) =0
\end{gather}
which immediately implies Eq.~\eqref{eqprop1}.
\end{proof}

Next we investigate the performance of the proposed approach, and establish guarantees on global performance achieved via  local
decisions dictated by Algorithm~\ref{AlgorithmOPT}.
We need some technical lemmas, and the notion of 
strongly absorbing graphs, and graph powers. 
\begin{defn}[Exact Power of Graph]
For a given graph $G=(V,E)$, the exact power $G^d$, for $d \in \mathbb{N}$, is a graph $(V,E')$, such that 
 $(q_i,q_j)$ is an edge in $G^d$, only if there exists a sequence of edges of length exactly $d$ from agent $q_i$ to agent $q_j$ in $G$.
 \end{defn}
\begin{defn}[Strongly Absorbing Graph]\label{defNSA}
 A finite  directed graph $G=(V,E)$ ($V$ is the set of agents and $E \subseteq V \times V$  the set of edges) is defined to be  strongly absorbing (SA), if:
\begin{enumerate}
 \item There are one or more absorbing agents, $i.e.$, $\exists A \subsetneqq V$, s.t. every agent in $A$ (non-empty) is absorbing.
\item There exists at least one sequence of edges from any agent to one of the absorbing agents in $A$.
\item If $E^d$ denotes the set of edges for the $d^{th}$ exact power of $G$, then, for distinct agents $q_i,q_j \in V$, 
\begin{gather}
 (q_i,q_j) \in E \Rightarrow  \forall d \in \mathbb{N}, \ (q_j,q_i) \notin E^d
\end{gather}
\end{enumerate}

\end{defn}
\begin{lem}[Properties of SA Graphs]\label{lemNSA}
Given a SA graph $G=(V,E)$, with $A \subsetneqq V$ the absorbing set:
\begin{enumerate}
\item The  power graph $G^d$ is SA for every $d \in \mathbb{N}$.
 \item $ q \notin A \Rightarrow  \exists q' \in V\setminus \{ q\} \ \textrm{s.t.} \ (q',q) \notin E$ 
\item $\exists d \in \mathbb{N} \left (  \forall q \in V\setminus A \left ( \exists q' \in A \left ( (q,q') \in E^d \right ) \right ) \right )$
\end{enumerate}
\end{lem}
\begin{proof} Statement 1 is immediate from Definition~\ref{defNSA}.
 Statement 2 follows immediately from noting:
\begin{gather*}
 q \notin A \Rightarrow \exists q'\in V\setminus \{ q\} \ \textrm{s.t.} \ (q,q') \in E \Rightarrow (q', q) \notin E
\end{gather*}
Statement 3 follows, since from each agent there is a path (length bounded by $\Crd(V)$) to a absorbing state.
\end{proof}

The performance of such control policies, and particularly the convergence time-complexity is closely related to the 
spectral gap of the induced Markov Chains. Hence we need to compute lower bounds on the spectral gap of the chains arising in the context of the proposed optimization, which (as we shall see later) have the 
strongly absorbing property. The following result  computes such a bound as a simple function of the non-unity diagonal entries of $\Pi$.
\begin{prop}[Spectral Bound]\label{propspectral}
 Given a n-state PFSA $G=(Q,\Sigma,\delta,\widetilde{\Pi})$ with a  strongly absorbing graph, the magnitude of non-unity eigenvalues of the transition  matrix $\Pi$ is bounded above by the maximum non-unity diagonal entry of $\Pi$.
\end{prop}
\begin{proof}
Without loss of generality, we assume that $G$ has a single absorbing state (distinct absorbing states can be merged without affecting non-unity eigenvalues). Now, $\mu$ is an eigenvalue of $\Pi$ iff $\mu^d$ is an eigenvalue of $\Pi^d, d \in \mathbb{N}$. From Lemma~\ref{lemNSA}:
\begin{itemize}
 \item[C1] $\exists \ell \in \mathbb{N}$ s.t. $\Pi^\ell$ has no zero entry in column corresponding to the absorbing state. Let $d_\star$ be the smallest such integer.
\item[C2] Every non-absorbing state has at least one zero element in the corresponding column of $\Pi^{d_\star}$.
\item[C3] Statements C1,C2  are true for any integer $d \geqq d_\star$.
\end{itemize}
We denote the column of ones as $\ones$, $i.e.$, $\ones = [1 \cdots 1]^T$ Since $\Pi^{d}$ is (row) stochastic, we have 
 $\Pi^{d} \ones =  \ones$. Hence, if $v$ is a left eigenvector for $\Pi^{d}$ with eigenvalue $\mu^{d}$, then:
\begin{gather}
v \Pi^{d} \ones = v\ones = \mu^{d} v \ones \Rightarrow (1-\mu^{d}) v \ones = 0
\end{gather}
implying that if $\mu^{d} \neq 1$, then $v\ones = 0$. Now we construct
 $\mathds{C} = [ C_1 \cdots C_n ]$, where $C_j = \min_j \Pi^{d}_{ij}$ (minimum column element). Considering $M=\Pi^{d} - \ones \mathds{C}$, we note:
\begin{gather}
( v\Pi^{d} = \mu^{d} v ) \wedge ( \mu^{d} \neq 1 )\Rightarrow v M = \mu^{d} v
\end{gather}
Recalling that stationary probability vectors (Perron vectors) of stochastic matrices add up to unity, we have:
\begin{gather}
 ( v\Pi^{d} =  v ) \Rightarrow v M = v - v\ones \mathds{C} = v - \mathds{C} 
\end{gather}
which, along with the fact that since $\mathds{C}$ is not a column of all zeros, implies that  an upper bound on the magnitudes of the eigenvalues of $M$ provides an upper bound on 
the magnitude of non-unity eigenvalues for $\Pi^{d}$. 
Now, invoking the 
Gerschgorin Circle Theorem~\cite{G31,V04}, we get:
\begin{gather}
 \vert \mu^{d} \vert \leqq 1 - \sum_j C_j = 1 - C_a \Rightarrow \vert \mu \vert \leqq \left ( 1 - C_a \right )^{\frac{1}{d}}
\end{gather}
where $C_a$ is the minimum column element corresponding to the  absorbing state. $1 - C_a$ is the maximum probability of not reaching the absorbing state after $d$ steps from any state, which is bounded above by
$ (a)^{d_1} ( b)^{d -d_1}$
 where $a$ is the maximum non-diagonal entry in $\Pi$ not going to the absorbing state, $b$ is the maximum of the non-unity diagonal entries in $\Pi$, and $d_1$ is a bounded integer. Since any sequence of non-selfloops is absorbed in a finite number of steps (strongly absorbing property), we have a finite bound for $d_1$. Hence we have:
\begin{gather}
 \vert \mu \vert \leqq \lim_{d \rightarrow \infty} a^{\frac{d_1}{d}} b^{1 - \frac{d_1}{d}} = b = \max_{q_i : \Pi_{ii} < 1} \Pi_{ii}
\end{gather}
This completes the proof.
\end{proof}

Next, we make rigorous our notion of policy performance, and near-global or $\epsilon$-optimality.
\begin{defn}[Policy Performance \& $\epsilon$-Optimality]\label{defperf} The performance vector $\rho^S$ of a given routing policy $S$  is the vector of agent-specific probabilities of a packet eventually reaching the target.
 A policy $U$ has \textit{Utopian performance} if its performance vector (denoted as $\rho^U$) element-wise dominates the one for  any arbitrary policy $S$, $i.e.$
$ \forall q_i \in Q^N ,\rho^U_i \geqq \rho^S_i$.
A policy $P$ has \textit{$\epsilon$-optimal} performance, if for  $\epsilon > 0$, we have:
\begin{gather}
 \vert \vert \rho^P - \rho^U \vert \vert_\infty \leqq \epsilon
\end{gather}
\end{defn}

For a chosen $\theta$, the limiting policy $P_\theta$ computed by Algorithm~\ref{AlgorithmOPT} results in element-wise maximization of the measure vector over all possible
supervision policies (where supervision is to be understood in the sense of the defined control philosophy). $\widehat{\nu}^\infty_\theta$ 
is related to the policy performance vector $\rho^{P_\theta}$ as follows. Selective disabling of the transitions dictated by the policy $P_\theta$ induces a controlled PFSA, which represents the optimally supervised network, for a given $\theta$. Let the optimized transition matrix  be $\Pi^\star_\theta$, and its Cesaro limit be $\Q^\star_\theta$. (Note:  $\Pi^\star_\theta$, $ \Q^\star_\theta$ are stochastic matrices.) Then:
\begin{gather}\label{eqperf}
 \forall q_i \in Q^N, \Q^\star_\theta \chi\big \vert_{i,(q_\Sn)} = \rho^{P_\theta}_i
\end{gather}
We would need to distinguish between the optimal measure vector  $\widehat{\nu}^\infty_{\theta'} $ (optimal for a given $\theta=\theta'$) computed by Algorithm~\ref{AlgorithmOPT}, and the one obtained by 
first computing $\widehat{\nu}^\infty_{\theta'} $ and then using the  PFSA  structure obtained in the process to compute the measure vector for some other value of $\theta=\theta''$.  These  two vectors may not be identical. 
\begin{notn}\label{not6}
In the sequel, we denote the vector obtained in the latter case as $\mudble{\theta'}{\theta''} $ implying that we have $ \mudble{\theta}{\theta} = \widehat{\nu}^\infty_{\theta}$.
\end{notn}
\begin{lem}\label{lem3} We have the following equalities:
\begin{subequations}
\begin{gather}
\lim_{\theta \rightarrow 0^+} \mudble{\theta'}{\theta} = \rho^{P_{\theta'}}\label{eqlemst2}\\
\lim_{\theta \rightarrow 0^+} \mudble{\theta}{\theta} = \rho^{U}\label{eqlemst3}
\end{gather}
\end{subequations}
\end{lem}
\begin{proof}
Recalling Eq.~\eqref{eqperf}, and noting that  for any PFSA with transition matrix $\Pi$ (with Cesaro limit $\Q$), we have $\lim_{\theta \rightarrow 0^+} \widehat{\nu}_\theta = \lim_{\theta \rightarrow 0^+}\theta \big [ \mathbb{I} - (1-\theta)\Pi \big ]^{-1} \chi = \Q \chi$, we have Eq.~\eqref{eqlemst2}.
In general, different choices of $\theta$ result in different disabling decisions, and hence different policies. However, since there is at most  a finite number of distinct policies for 
a finite network, there must exist a $\theta_\star$ such that for all choices $0 < \theta \leqq \theta_\star$, the policy remains unaltered (although the measure values may differ). Since, executing the optimization with vanishingly small $\theta$  yields a performance vector identical (in the limit) with the optimal measure vector  element-wise dominating the 
one for any arbitrary policy, the policy obtained for $0<\theta\leqq\theta_\star$ has Utopian performance. Hence:
\begin{gather}
 \lim_{\theta \rightarrow 0^+} \mudble{\theta}{\theta} = \lim_{\theta \rightarrow 0^+} \mudble{\theta_\star}{\theta} = \rho^{P_{\theta_\star}} = \rho^{U}
\end{gather}
This completes the proof.
\end{proof}

Computation of the critical $\theta_\star$ is non-trivial from a distributed perspective, although centralized approaches have been reported~\cite{CR07}. Thus it is hard to guarantee Utopian performance in Algorithm~\ref{AlgorithmOPT}. Also, $\theta_\star$ may be too small resulting in an unacceptably poor convergence rate.
 Nevertheless, we will show that, given any $\epsilon > 0$, one can choose $\theta$ to guarantee $\epsilon$-optimal performance of the limiting policy in the sense of Definition~\ref{defperf}. We would  need the following  result.
\begin{lem}\label{lemtechnical}
 Given any PFSA, with transition matrix $\Pi$ and corresponding Cesaro limit $\Q$, and  $\mu$ being a non-unity eigenvalue of $\Pi$  with  maximal magnitude, we have:
\begin{subequations}
\calign[1pt]{
&\big \vert \big \vert \theta \big [ \mathbb{I} - (1-\theta)\Pi \big ]^{-1}- \Q \big \vert\big\vert_\infty \leqq \frac{\theta}{1-\vert \mu\vert }\label{eqclaim31}\\
 &\big\vert\big \vert\nu_{(\theta,\theta) } - \lim_{\theta' \rightarrow 0^+}\nu_{(\theta,\theta')}  \big\vert\big\vert_\infty \leqq \frac{\theta\vert \vert \chi \vert \vert_\infty}{1-\vert \mu\vert } \label{eqclaim33}
}
\end{subequations}
\end{lem}
\begin{proof}
Denoting $M=\big [ \mathbb{I} - (1-\theta)\Pi \big ]^{-1}- \frac{1}{\theta}\Q$, 
\begin{align}\textstyle
 M = & [\mathbb{I} - (1-\theta)\Pi ]^{-1}  - \Q \sum_{k=0}^\infty (1-\theta)^k
 = \sum_{k=0}^\infty (1-\theta)^k (\Pi -\Q)^k -\Q  
\notag\\
= & [\mathbb{I} - (1-\theta)(\Pi - \Q)]^{-1} -\Q    \notag
\end{align}
We note, that if $u$ is a left eigenvector of $\Pi$ with unity eigenvalue, then $u\Q = u$. Also, if the eigenvalue corresponding to $u$ is strictly within the 
unit circle, then $u\Q = 0$. After a little algebra, it follows that  if $u$ is the left eigenspace (denoted as $E(1)$) corresponding to unity eigenvalues of $\Pi$, then 
$uM=0$, otherwise, $uM=\frac{1}{1-(1-\theta)\mu} u$, where $\mu$ is a non-unity eigenvalue for $\Pi$. Invoking the definition of induced matrix norms, and noting $\vert\vert A\vert \vert_\infty =\vert\vert A^T\vert \vert_1$ for any square matrix $A$:
\begin{gather}
 \vert \vert M \vert \vert_\infty = \max_{\vert \vert u \vert \vert_1 = 1} \vert \vert uM \vert \vert_1 = \max_{\vert \vert u \vert \vert_1 = 1 \wedge u \notin E(1)} \vert \vert uM \vert \vert_1
\end{gather}
We further note that since $[\mathbb{I} - (1-\theta)(\Pi - \Q)]^{-1}$ is guaranteed to be invertible~\cite{CR07}, its eigenvectors form a basis, implying:
\begin{gather}\textstyle
 u = \sum_j c_j u^j, \ \textrm{with} \ \left \vert \left\vert \sum_j c_j u^j\right \vert\right \vert_1 = 1
\end{gather}
where $u^j$ are eigenvectors of $[\mathbb{I} - (1-\theta)(\Pi - \Q)]^{-1}$ with non-unity eigenvalues, and $c_j $ are complex coefficients.
 An upper bound for $\vert \vert M \vert \vert_1$ can be now computed as:
\begin{gather*}\textstyle
 \vert \vert M \vert \vert_\infty \leqq \frac{1}{1-(1-\theta)\vert \mu\vert } \left \vert \left\vert \sum_j c_j u^j\right \vert\right \vert_1 = \frac{1}{1-(1-\theta)\vert \mu\vert } \leqq  \frac{1}{1-\vert \mu\vert}
\end{gather*}
where $\mu$ is a  non-unity eigenvalue for $\Pi$ with maximal magnitude. This establishes Eq.~\eqref{eqclaim31}. 
Finally, noting:
\begin{gather*}
 \nu_{(\theta,\theta) } - \lim_{\theta' \rightarrow 0^+}\nu_{(\theta,\theta')} = \big (\theta[\mathbb{I} - (1-\theta)\Pi ]^{-1}  - \Q \big )\chi
\end{gather*}
establishes Eq.~\eqref{eqclaim33}.
\end{proof}

The next proposition the  key result relating a specific choice of  $\theta$ to guaranteed $\epsilon$-optimal performance.
\begin{prop}[Global $\epsilon$-Optimality]\label{propglobal}
 Given any $\epsilon > 0$, choosing 
$ \theta = \left.\epsilon\right/m^2 \ \textrm{where} \ m=\max_{q\in Q}\Crd(\mathcal{N}(q))$
guarantees that the limiting policy computed by Algorithm~\ref{AlgorithmOPT} is  $\epsilon$-optimal in the sense of Definition~\ref{defperf}.
\end{prop}
\begin{proof}
We observe that limiting measure values $\nuinf{ \vert_i} =  \nu^\star_\theta \vert_i$ computed by Algorithm~\ref{AlgorithmOPT} can be represented by convergent sums of the form ($a_{ij}$: non-negative reals):
\begin{gather}\label{eqnondecreasing}
\forall q_i \in Q^N, \ \nuinf{ \vert_i} = \sum_{j=1}^\infty a_{ij} (1-\theta)^{j}
\end{gather}
implying that for each $q_i\in Q$, $\mudble{\theta}{\theta_1} \vert_i$ (See Notation~\ref{not6}) is a monotonically decreasing function of $\theta_1$ in the domain $[0,\theta]$.
We note that if the following statement:
\begin{gather*}
\forall q_i,q_j \in Q^N,   \nuinf{\vert_i} > \nuinf{\vert_j} \Rightarrow \forall \theta_1 \leqq \theta, \ \mudble{\theta}{\theta_1} \vert_i > \mudble{\theta}{\theta_1} \vert_j
\end{gather*}
is true, then we have Utopian performance for  policy $P_\theta$, $i.e.$, $\rho^{P_\theta} = \rho^U$. Hence, if $\rho^{P_\theta} \neq \rho^U$, then we must have:
\begin{gather*}
 \exists \theta_2 < \theta, \exists q_i,q_j \in Q^N, \big ( \nuinf{\vert_i} > \nuinf{\vert_j} \big )\wedge   \big ( \mudble{\theta}{\theta_1} \vert_i > \mudble{\theta}{\theta_1} \vert_j\big )
\end{gather*}
upon which Eq.~\eqref{eqnondecreasing}, along with the bound established in Eq.~\eqref{eqclaim31}, guarantees that if $q_i,q_j$ are agents (in consecutive order) that satisfy the above statement, then:
\begin{gather}\label{eqbnd11}
 \lim_{\theta_1 \rightarrow 0^+} \left( \mudble{\theta}{\theta_1} \vert_i - \mudble{\theta}{\theta_1} \vert_j \right)\leqq \beta_\theta \theta
\end{gather}
where $\beta_\theta = \frac{1}{1 - \vert \mu \vert }$, with $\mu$ being a maximal non-unity eigenvalue of the transition matrix of the PFSA computed by Algorithm~\ref{AlgorithmOPT} at $\theta$.
Next we claim:
\begin{gather}
\forall \theta' \in (0, \theta],  \ \vert\vert \nuinf[\theta']{} - \mudble{\theta}{\theta'} \vert\vert_\infty \leqq m^2\theta
\end{gather}
We observe that, for any $\theta'$, the optimal policy $P_{\theta'}$ can be obtained by beginning with the PFSA induced  by $P_\theta$ (which is the optimal policy at $\theta$), and then executing the 
centralized iterative approach~\cite{CR07}, resulting in a sequence of \textit{element-wise non-decreasing} measure vectors converging to the 
optimal $\nuinf[\theta']{}$:
\begin{gather}
 \mudble{\theta}{\theta'} = \nu^{[0]}_{\theta'} > \nu^{[1]}_{\theta'}> \nu^{[2]}_{\theta'}>\cdots \nu^{[k^\star]}_{\theta'}=\nuinf[\theta']{}
\end{gather}
where $\nu^{[k]}_{\theta'}$ is the  vector obtained after the $k^{th}$ iteration, and $k^\star < \infty$ is the number of required iterations. 
Since, $
 \nu^{[k]}_{\theta'} = \theta'\big [ \mathbb{I} - (1-\theta')\Pi^{[k]}\big]^{-1}\chi
$, where  the transition matrix after $k^{th}$ iteration is $\Pi^{[k]}$ and setting $\Delta^{[k]}_{\theta'}= \nu^{[k]}_{\theta'} - \mudble{\theta}{\theta'}$ we have:
\begingroup\setlength\belowdisplayskip{0pt}
\begin{align}
 &\Delta^{[k]}_{\theta'} =  (1-\theta')\big [ \mathbb{I} - (1-\theta')\Pi^{[k]}\big]^{-1} (\Pi^{[k]} - \Pi^{[0]}) \mudble{\theta}{\theta'} \notag \\ 
  &= {\textstyle\frac{1-\theta'}{\theta'}}\big \{ {\red \underbrace{\black \theta' \big [ \mathbb{I} - (1-\theta')\Pi^{[k]}\big]^{-1} }_{\red \mathds{B}^{[k]}_{\theta'}}}\big \}\big \{ {\red \underbrace{\black(\Pi^{[k]} - \Pi^{[0]}) \mudble{\theta}{\theta'}}_{\red \omega^{[k]}_{\theta'}}}\big \}\notag
\vspace{-15pt} 
\end{align}
\endgroup
For $q_i \in Q$, let ${\upd}_i^{(0\rightarrow k)}$ be the set of transitions $(q_i \xrightarrow{\sigma} q_j)$, which are updated ($i.e.$ enabled if disabled or vice verse) to go from the configuration corresponding to $\nu^{[0]}_{\theta'}$ to the one corresponding to $\nu^{[k]}_{\theta'}$.
We note that:
\begin{gather*}
 \upd_i^{(0\rightarrow k)} = \left ( \upd_i^{(0\rightarrow 1)} \cap \upd_i^{(0\rightarrow k)} \right ) \bigcup \mathscr{W}
\end{gather*}
where $\mathscr{W} = \upd_i^{(0\rightarrow k)} \setminus \left ( \upd_i^{(0\rightarrow 1)} \cap \upd_i^{(0\rightarrow k)} \right )$. 
The $i^{th}$ row of  $\Pi^{[1]}$ is obtained from $\Pi^{[0]}$~\cite{CR07} by disabling controllable transitions $q_i \xrightarrow{\sigma} q_j$ if $\nu^{[0]}_{\theta'} \vert_j > \nu^{[0]}_{\theta'} \vert_i$ (and enabling otherwise), and each such update leads to a positive contribution in the corresponding row of $\omega^{[1]}_{\theta'}$. It follows that 
updating any transition $t \equiv (q_i \xrightarrow{\sigma} q_j) \in \left ( \upd_i^{(0\rightarrow 1)} \cap \upd_i^{(0\rightarrow k)} \right )$ leads to a positive contribution to $\omega^{[k]}_{\theta'}\vert_i$, given by:
\begin{gather}
 C_t = \Pitilde (q_i ,\sigma) \left \vert \nu^{[0]}_{\theta'}\big \vert_i -\nu^{[0]}_{\theta'}\big \vert_j \right \vert  
\end{gather}
Every  $t' \equiv (q_i \xrightarrow{\sigma'} q_k) \in \mathscr{W}$ causes a negative contribution to $\omega^{[k]}_{\theta'}\vert_i$, given by:
\cgather{
 C_{t'} = -\Pitilde (q_i ,\sigma') \left \vert \nu^{[0]}_{\theta'}\big \vert_i -\nu^{[0]}_{\theta'}\big \vert_k \right \vert  \\
\mspace{0mu} \textrm{implying that:} \mspace{10mu} 
 \omega^{[k]}_{\theta'}\vert_i \leqq \mspace{0mu} \sum_{\mspace{0mu} r \in \left ( \upd_i^{(0\rightarrow 1)} \bigcap \upd_i^{(0\rightarrow k)}\right ) } \mspace{-0mu} C_r  \mspace{0mu} \\
\Rightarrow \omega^{[k]}_{\theta'}\vert_i \leqq \sum_{\sigma \in \Sigma} \Pitilde(q_i,\sigma) \beta_\theta \theta  = \beta_\theta \theta  \mspace{20mu} \textrm{(See Eq.~\eqref{eqbnd11})}\notag
}
Since the rows corresponding to the absorbing states have no controllable transitions, absorbing states must remain absorbing through out the iterative sequence, and the corresponding entries in $\omega^{[k]}_{\theta'}$ for all $k\in\{0,\cdots,k^\star\}$ are strictly $0$. It follows:
\begin{align}
 &\omega^{[k]}_{\theta'}\vert_i = \left \{ \begin{array}{ll}
0 & ,\textrm{if $q_i$ is absorbing}\\
       \in [0,  \beta_\theta\theta] & ,\textrm{otherwise }
                              \end{array}
\right.\label{eq52}
\end{align}
Stochasticity of $\mathds{B}^{[k]}_{\theta'}$ implies that in the limit $\theta' \rightarrow 0^+$, $\mathds{B}^{[k]}_{\theta'}$ converges to the 
Cesaro limit of $\mathds{B}^{[k]}_{\theta'}$. Applying Lemma~\ref{lemtechnical}:
\begin{gather}
\big \vert \big \vert \mathds{B}^{[k]}_{\theta'} - \lim_{\theta' \rightarrow 0^+}\mathds{B}^{[k]}_{\theta'}\big \vert \big \vert_\infty \leqq \frac{\theta'}{1-\vert \mu_{\theta'}\vert } \triangleq \beta_{\theta'}\theta'
\end{gather}
where $\mu_{\theta'}$ is a non-unity eigenvalue for $\mathds{B}^{[k]}_{\theta'}$ with maximal magnitude. Using the invariance of the absorbing state set,
and observing that the Cesaro limit $\lim_{\theta' \rightarrow 0^+}\mathds{B}^{[k]}_{\theta'}$ has strictly zero columns corresponding to non-absorbing states, we conclude:
\begin{gather*}
 \forall \theta' \in(0,\theta], \ \Delta^{[k]}_{\theta'}\vert_i \leqq \frac{1-\theta'}{\theta'}\beta_{\theta'}\theta'\beta_{\theta}\theta \leqq \beta_{\theta'}\beta_{\theta}\theta
\end{gather*}
It is easy to see that the PFSA induced by $P_\theta$ is strongly absorbing (Definition~\ref{defNSA}), and so is each one obtained in the iteration. 
Also, the virtual states in our network model have no controllable transitions, and have no self-loops. Physical agents can have self-loops arising from disablings;
but for a non-absorbing agent with at most $m$ neighbors, the self-loop probability is bounded by $(m-1)/m$, which then implies $\beta_{\theta'},\beta_\theta \leqq \frac{1}{1 - (m-1)/m} = m$ (Proposition~\ref{propspectral}).
Hence:
\begin{gather}
 \forall \theta' \in (0,\theta], \ \vert \vert \Delta^{[k]}_{\theta'}\vert\vert_\infty \leqq m^2 \theta
\end{gather}
Thus, if we choose $\theta = \epsilon / m^2$, we can argue:
\begin{align*}
  &\forall k \in \{0,\cdots,k^\star\}, \  \forall \theta' \in (0, \theta], \ \vert \vert\Delta^{[k]}_{\theta'} \vert\vert_\infty \leqq  \epsilon  \\
 \Rightarrow & \lim_{\theta' \rightarrow 0^+} \left \vert \left \vert \widehat{\nu}^{\infty}_{\theta'} - \mudble{\theta}{\theta'} \right \vert \right \vert_\infty \leqq \epsilon \\ 
\Rightarrow  & \left \vert \left \vert \lim_{\theta' \rightarrow 0^+}\widehat{\nu}^{\infty}_{\theta'} - \lim_{\theta' \rightarrow 0^+}\mudble{\theta}{\theta'} \right \vert \right \vert_\infty \leqq \epsilon  \ \left ( \textrm{\small \sffamily \txt{{Continuity}\\ {of norm}}} \right )\notag \\ 
 \Rightarrow & \left \vert \left \vert \rho^U - \rho^{P_\theta} \right \vert \right \vert_\infty \leqq \epsilon \ \left ( \textrm{\small \sffamily {Using Lemma~\ref{lem3}}} \right )\notag
\end{align*}
which completes the proof.
\end{proof}

Once we have guaranteed convergence to a $\epsilon$-optimal policy, we need to compute asymptotic bounds on the time-complexity of route convergence, $i.e.$, how long it takes to 
converge to the limiting policy so that the local routing decisions no longer fluctuate. In practice, the convergence time is dependent on the network delays, the degree to which the agent updates are synchronized $etc.$, and is difficult to estimate. In this paper, we neglect such effects to obtain an asymptotic estimate in the perfect situation. This allows us to quantify the  dependence of the convergence time on key parameters such as $N$, $m$ and $\epsilon$. Future work will address situations where such possibly implementation-dependent effects are explicitly considered resulting in  potentially smaller convergence rates.
\begin{prop}[Asymptotic Runtime Complexity]\label{propcomplex}With no communication delays and assuming synchronized updates,  convergence time $\Tc $ to $\epsilon$-optimal operation for a network of $N$ physical agents and maximum $m$ neighbors, satisfies:
\begin{gather*}
  \Tc  = O\left (  \frac{Nm^2}{\epsilon (1-\gamma_\star)}\right ) \\ \textrm{where $\gamma_\star$ is a lower bound on failure probabilities} 
 \end{gather*}
\end{prop}
\begin{proof}
 Synchronized updates imply that we can assume  the following recursion:
\begin{subequations}
\begin{align}
 \widehat{\nu}^{[1]}_\theta  &= \boldsymbol{0} \ \textrm{(Zero vector)}\\
 \widehat{\nu}^{[k+1]}_\theta  &= (1-\theta)\Pi^{[k]} \widehat{\nu}^{[k]}_\theta + \theta\chi
\end{align}
\end{subequations}
which can be used to obtain the upper bound:
\begin{gather}
\big \vert \big \vert \widehat{\nu}^{\infty}_\theta - \widehat{\nu}^{[k]}_\theta  \big \vert \big \vert_\infty \leqq (1-\theta)^k
\end{gather}
implying that after $k$ updates, each agent is within $(1-\theta)^k$ of its limiting value. Denoting the smallest difference of measures as $\Delta_\star$, we note that 
$(1-\theta)^k\leqq \Delta_\star$ would guarantee that no further route fluctuation occurs, and the network operation will be $\epsilon$-optimal from that point onwards. To estimate $\Delta_\star$, we note that 1) comparisons cannot be made for values closer than the machine precision $M_0$, and 2) 
the lowest possible non-zero measure in the network occurs at the network boundaries if we assume the worst case scenario in which the failure probability is always $\gamma_\star$. We recall  
the measure of a agent is  the sum of the measures of all paths initiating from the particular agent and terminating at the target. Also, note that 
any such path accumulates a multiplicative factor of $(1-\theta)^2(1-\gamma_\star)$ in each hop. In the worst case a given agent is $N$ hops away, and has a single path to the target, implying that the smallest non-zero measure of any agent is bounded below by $( (1-\theta)^2(1-\gamma_\star))^N$. Hence:
\begin{gather}
 \Delta_\star \geqq M_0 \left( (1-\theta)^2(1-\gamma_\star)\right )^N
\end{gather}
and hence a sufficient condition for convergence is:
\begin{align}
 &(1-\theta)^k = M_0 \left( (1-\theta)^2(1-\gamma_\star)\right )^N 
\Rightarrow (1-\theta)^{(k-2N)} = M_0 (1-\gamma_\star)^N \notag\\
\Rightarrow &k = 2N + \frac{\log M_0}{\log (1-\theta)} + N \frac{\log (1-\gamma_\star)}{\log (1-\theta)} 
\end{align}
Treating $M_0$ as a constant, we have 
 $ \frac{\log M_0}{\log (1-\theta)} = O\left (\frac{1}{\theta}\right )$.
Since $\theta$ must be small for near-optimal operation and considering the worst case $\gamma_\star \ll 1$, we have: 
\begin{align}
&(1-\theta)^{k_1} = 1 - \gamma_\star \mspace{20mu} \textrm{where }k_1 \triangleq \frac{\log (1-\gamma_\star)}{\log (1-\theta)} \notag\\
\Rightarrow &(1-k_1\theta) \simeq 1 - \gamma_\star \Rightarrow  k_1 \theta = \gamma_\star 
\Rightarrow k_1 = \frac{\gamma_\star}{\theta} \notag\\ \Rightarrow & k_1 \simeq \frac{1}{\theta (1 - (1-\gamma_\star))^{-1}}\Rightarrow k_1  =O\left( \frac{1}{\theta( 1-\gamma_\star)}\right)\notag\\
\Rightarrow & k = O\left(N+ \frac{1}{\theta} +  \frac{N}{\theta(1-\gamma_\star)}\right ) = O\left( \frac{N}{\theta(1-\gamma_\star)}\right )\notag\\
\Rightarrow & k = O\left( \frac{Nm^2}{\epsilon(1-\gamma_\star)}\right ) \mspace{20mu} \mathsf{(Using \ Proposition~\ref{propglobal})} \notag
\end{align}
Thus we have $\Tc  = O(k)$, which completes the proof.
\end{proof}

It follows from Proposition~\ref{propcomplex} that for constant $\epsilon$ and $\gamma_\star$, and large networks with relatively smaller number of local neighbors such that $N \gg m$, we will have 
$ \Tc  =O(N)$. 
{\itshape Detailed simulation, on the other hand, indicates that this bound is not tight, as illustrated in Figure~\ref{figcomplexb}, where we see a logarithmic dependence instead.}
\begin{table}[t]
\caption{Instantaneous Agent Data Table}\label{tableAgentdata}
\vspace{-5pt}
\centering
\begin{tabular}{|p{.05in}|c|c|c|c|}\hline
\sffamily \txt{Id.}&\sffamily \txt{Neighbor  \#}& \sffamily \txt{Current \\ Measure} & \sffamily \txt{$\phantom{^1}$Failure \\ $\phantom{_1}$Probability} & \sffamily \txt{Forwarding \\ Decision }\\\hline
{$I_1$}& \BRed {{\scshape (Self)} $1$} & \bf \BRed {$\nu_0$} & \bf \BRed {$d_0=0$} & \bf \BRed {$0$}\\ \hline
$\vdots$&$\vdots$ & $\vdots$ & $\vdots$ & $\vdots$ \\ \hline
{$I_m$}&$m$ & $\nu_m$ & $d_m$ & $1$\\ \hline
 \end{tabular}
\end{table}
The stationary policy computed for the frozen swarm has some additional properties, as we establish next.

\begin{prop}[Properties]\label{propchar}
     The limiting frozen policy is stationary and  has the following additional properties:
     \begin{enumerate}
          \item is loop-free
 \item  is the unique loop-free policy that disables the smallest set of transitions among all policies which induce the same measure vector for a given $\theta$. 
\end{enumerate}
\end{prop}
\begin{proof}
Stationarity is obvious.
(1) Absence of loops follows immediately from noting that, 
 in the limiting policy, a controllable transition $q_i \rightarrow q^v_{(ij)}$ is enabled if and only if $q^v_{(ij)}$ has a limiting measure strictly greater than that of  $q_i$, implying that any sequence of transitions (with no consecutive repeating states) goes to either the dump  or the target in  a finite number of steps.

(2) follows directly from the uniqueness and the maximal permissivity property of optimal policies computed by language measure-theoretic optimization (See \cite{CR07}).
\end{proof}

One can  easily  tabulate the  data that needs to be maintained at each agent (See Table~\ref{tableAgentdata}). In particular, each agent needs to know the unique network id. of each neighbor that it can communicate with (Col. 1), and their current measure values (Col. 3). The failure probabilities for communicating from self to each of those neighbors must be maintained as well, for the purpose of carrying out the distributed updates (Col. 4). The forwarding decision is a neighbor-specific Boolean value (Col. 5), which is set to $1$ if the neighbor currently has a strictly higher measure than self, and $0$ otherwise. The packets are then forwarded by randomly choosing (in an equiprobable manner) between  the enabled neighbors, $i.e.$, the ones with a \textbf{true} forwarding decision.
Note that this agent data  updates when the measures of the neighbors change (Col. 3), or the failure probabilities (Col. 4) update. However, changes in the measures may not necessarily reflect a change in the forwarding decisions. Also, note that the routing is inherently probabilistic, (due to the possibility that multiple enabled neighbors may exist for a given agent). Furthermore, the optimal policy disables navigation decisions to as few neighbors as possible for a specified $\theta$  (Proposition~\ref{propchar}),  and hence exploits available alternate routes  in an optimal manner, thereby reducing congestion.

\section{Simultaneous Navigation \& Decision Optimization: The "Unfrozen" Case}\label{secMob}
\begin{notn}[Best Neighbors]
For a fixed agent $q_i$, the set of neighboring  agents having maximal measure at operation time $t$ is denoted as $\mathds{B}_i(t)$. Furthermore, let $\mathfrak{b}^\star_i(t)$ denote a randomly chosen maximal agent, towards which $q_i$ has decided to move at time $t$.
\end{notn}
\begin{notn}[Swarm Configuration]
     Denote the vector of positional coordinates of the agents at time $t$ as $\mathcal{P}(t)$.
\end{notn}

\begin{defn}[Movement Mechanism]\label{defMOV}
The positional update mechanism of the swarm can now be concretely stated as:
\begin{enumerate}
\item[\bf C1] After step (a4) in Algorithm~\ref{AlgorithmOPT}, for each $q_i$ choose a maximal agent  $\mathfrak{b}^\star_i(t)$, and move towards $\mathfrak{b}^\star_i(t)$ at a constant velocity, with the following restriction.
\item[\bf C2] If there exists $q_j$ such that $q_i = \mathfrak{b}^\star_j(t)$, then make sure that the 
distance from $q_j$ is within the communication radius.
\end{enumerate}
\end{defn}
\begin{defn}[Process $\mathscr{R}_{v_s}(t,\mathcal{P}(t'))$]
The stated movement mechanism induces a sequence of swarm configurations as a function of time $t$ denoted by $\mathscr{R}_{v_s}(t,\mathcal{P}(t'))$, which is understood to be the achieved vector of positional coordinates of the agents as a function of time $t\geqq t'$, beginning with the initial configuration $\mathcal{P}(t')$ at time $t'$, with the constant update velocity  $v_s > 0$.
\end{defn}
\begin{defn}[The Ideal Update Process]
 Assuming that the distributed route convergence for the frozen swarm occurs instantaneously, let $\mathscr{N}_{\Delta',v_s}(t,\mathcal{P}(t'))$
denote the vector of position coordinates of the agents at time $t\geqq t'$, obtained as a result of the following sequential operation, initiated with the swarm configuration $\mathcal{P}(t')$ at time $t'$:
\begin{enumerate}
     \item Freeze swarm
     \item Optimize routes via Algorithm~\ref{AlgorithmOPT} (assumed to occur instantaneously for this definition only)
     \item For time $\Delta'$, move each agent $q_i$ towards its  best neighbor (with some form of tie-breaking if required) with a constant velocity $v_s$.
     \item Go to step 1.
\end{enumerate}
Then the ideal update process is defined as the sequence of swarm configurations (as a function of the operation time $t$) given by:
\cgather[4pt]{
\mathscr{I}_{v_s}(t,\mathcal{P}(t')) = \lim_{\Delta' \rightarrow 0^+} \mathscr{N}_{\Delta',v_s}(t,\mathcal{P}(t')) 
}
\end{defn}
$\mathscr{I}_{v_s}(t,\mathcal{P}(t'))$ has the following immediate properties:
\begin{enumerate}
     \item At any point in time $t\geqq t'$, the routing policy in effect is globally $\epsilon$-optimal in the sense defined in the preceding section.
     \item Infinitesimal  updates at each time $t$ occur according to such $\epsilon$-optimal policies.
\end{enumerate}
Denoting $\mathscr{I}_{v_s}(t,\mathcal{P}(t'))\vert_i$ and $\mathscr{R}_{v_s}(t,\mathcal{P}(t'))\vert_i$ as the positional coordinates of the agent $q_i$ at time $t$ for the respective update processes, and $\mathcal{P}_\Sn$ as the positional coordinate of the target, we have the following  convergence results.
\begin{prop}[Convergence Of Swarm Trajectories]
For any initial configuration $\mathcal{P}(0)$ at time $t=0$, each agent eventually converges to the target, $i.e.$,
 \cgathers[2pt]{
 \forall q_i \in Q, \widehat{\nu}_\theta\vert_i(0) > 0 \Rightarrow \left \{\begin{array}{l}
 \forall \mathcal{P}(0), \forall q_i \in Q, \lim_{t \rightarrow \infty}\vert\vert\mathscr{I}_{v_s}(t,\mathcal{P}(0))\vert_i-\mathcal{P}_\Sn\vert\vert
 =0\\
 \forall \mathcal{P}(0), \forall q_i \in Q, \lim_{t \rightarrow \infty}\vert\vert\mathscr{R}_{v_s}(t,\mathcal{P}(0))\vert_i-\mathcal{P}_\Sn\vert\vert
 =0\end{array}\right.} 
\end{prop}
\begin{proof}
We consider the two processes $\mathscr{N}_{\Delta',v_s}(t,\mathcal{P}(0))$, for some $\Delta' >0$ and $\mathscr{R}_{\Delta',v_s}(t,\mathcal{P}(0))$. We note that  condition \textbf{C2} in Definition~\ref{defMOV} is automatically satisfied for $\mathscr{N}_{\Delta',v_s}(t,\mathcal{P}(0))$, since the distance between agent $q_i$ and $\mathfrak{b}_i^\star(t)$ is guaranteed to be non-decreasing if the agents move with a constant velocity $v_s$. This immediately implies that no agent in the swarm gets disconnected, and it follows that:
\cgather[2pt]{
\widehat{\nu}_\theta\vert_i(0) > 0 \Rightarrow \forall t > 0, \widehat{\nu}_\theta\vert_i(t) > 0 
}
since it is given that $\forall q_i \in Q, \widehat{\nu}_\theta\vert_i(0) > 0$ implying that at least one sequence of hops from agent $q_i$ to $q_\Sn$ of the form $\{q_i,q_{i'},\cdots, q_\Sn\}$ exists at time $t=0$, such that 
\cgather[2pt]{
\widehat{\nu}_\theta\vert_i \leqq \widehat{\nu}_\theta\vert_{i'} \leqq \cdots \leqq \widehat{\nu}_\theta\vert_\Sn=1
\label{eqagpath}}
and hence at least one such sequence is guaranteed to exist for all $t>0$. Let $h_t(q_i)\in \mathbb{N}$ be the minimum length of such a sequence from $q_i$ at time $t$. Since it is given that $\forall q_i \in Q, \widehat{\nu}_\theta\vert_i(0) > 0$, we have:
\cgather[2pt]{
\max_{q_i \in Q}h_0(q_i) \leqq \Crd(Q)
}
We note    that all agents $q_i$ with $h_0(q_i)=1$ are direct neighbors of the target, and hence simply move towards the latter at a constant velocity $v_s$ for all times until convergence. Let $t'$ be the time within which all such agents  do converge to the target. Then, it follows that:
\cgather[2pt]{
\max_{q_i \in Q}h_{t'}(q_i) \leqq \max_{q_i \in Q}h_0(q_i)  - 1
}
By continually applying the above argument, we obtain a sequence of times $\{t',t'',\cdots\}$, such that:
\cgather[2pt]{
\Crd(Q) \geqq \max_{q_i \in Q}h_0(q_i) > \max_{q_i \in Q}h_{t'}(q_i) > \max_{q_i \in Q}h_{t''}(q_i) > \cdots
}
Finite size of the swarm, and the fact that the above argument applies for all $\Delta' > 0$, then implies the desired result.
\end{proof}

$\mathscr{I}_{v_s}(t,\mathcal{P}(0))$ cannot be directly implemented in practice (due to the requirement of sequential freezing and instantaneous route optimization). However, it allows us to compute the performance of implementable policies by comparing how  close the achieved  sequence of swarm configurations is  to the ideal process. Note the in spite of convergence, the ideal process $\mathscr{I}_{v_s}(t,\mathcal{P}(0))$ differs significantly  in definition from $\mathscr{R}_{v_s}(t,\mathcal{P}(0))$, and we need to establish that the latter is in some meaningful sense close to the former. We need the following definition, and a notion of convergence rate.
\begin{defn}[Path Lengths To Target]\label{defhtilde}
     Recall that $\widehat{\nu}_\theta\vert_i(t) > 0$ implies that at least one sequence of hops from agent $q_i$ to $q_\Sn$ of the form $\{q_i,q_{i'},\cdots, q_\Sn\}$ exists at time $t$, such that 
\cgather[2pt]{
\widehat{\nu}_\theta\vert_i(t) \leqq \widehat{\nu}_\theta\vert_{i'}(t) \leqq \cdots \leqq \widehat{\nu}_\theta\vert_\Sn=1}
and $h_t(q_i)$ is the  minimum length of such a hop sequence from $q_i$ at time $t$. We define $\widetilde{h}_t(q_i)$ as the physical piecewise length of such a path (denoted by the indices of the agent sequence for simplified notation $i.e.$ writing $q_i$ as $i$)
\cgather[2pt]{S=\{i={j_{1}},j_2,\cdots,{j_{r-1}},{j_{r}},\cdots, {j_{h_t(q_i)}}=\Sn\}} as follows: 
\cgather[2pt]{
\widetilde{h}_t(q_i) = \sum_{j_{r} =j_2}^{j_r=j_{h_t(q_i)}} \left \vert \left \vert\mathcal{P}(t)_{{j_{r-1}}}-\mathcal{P}(t)_{{j_{r}}}\right \vert \right \vert \label{eqhtilde}
}
\end{defn}
\begin{prop}[Convergence Rate]\label{propconvrate}
     The swarm trajectories converge to the target at an exponential rate for both $\mathscr{I}_{v_s}(t,\mathcal{P}(0))$ and $\mathscr{R}_{v_s}(t,\mathcal{P}(0))$, $i.e.$, we have:
     \cgather[4pt]{
     \forall q_i \in Q, \widetilde{h}_t(q_i) \leqq \widetilde{h}_0(q_i) e^{-(v_s/ R_c) t}
     }
     where $R_c > 0$ is the specified constant communication radius.
\end{prop}
\begin{proof}
     We note that in Eq.~\eqref{eqhtilde}, agent $q_{j_r}\in\mathds{B}_{j_{r-1}}(t)$. Without loss of generality, we assume that $q_{j_r} = \mathfrak{b}_{j_{r-1}}^\star(t)$. Then it follows that:
     \cgather[2pt]{
     \frac{\mathrm{d}}{\mathrm{d}t}\widetilde{h}_t(q_i)= -h_t(q_i)v_s\label{eqdiff1}
     }
     Next, we note the bound:
     \cgather[2pt]{
     \widetilde{h}_t(q_i) \leqq R_c h_t(q_i) \Rightarrow -h_t(q_i) \leqq -\frac{1}{R_c}\widetilde{h}_t(q_i)
     }
     Using in Eq.~\eqref{eqdiff1}, we obtain:
     \cgather[2pt]{
     \frac{\mathrm{d}}{\mathrm{d}t}\widetilde{h}_t(q_i)\leqq -\left (\frac{v_s}{R_c}\right )\widetilde{h}_t(q_i)\label{eqdiff2}
     }
     which completes the proof.
\end{proof}
\begin{defn}[Swarm Diameter]
     The swarm diameter $\mathds{D}_t(Q)$ is defined as:
     \cgather[2pt]{
     \mathds{D}_t(Q) = 2\max_{q_i \in Q} \left \vert\left \vert \mathcal{P}_{q_i}(t) - \mathcal{P}_\Sn \right \vert \right \vert 
     }
\end{defn}

\begin{cor}[Corollary To Proposition~\ref{propconvrate}]\label{corconvrate}
 $\mathds{D}_t(Q) \leqq 2\mathds{D}_0(Q) e^{-(v_s/ R_c) t}$
\end{cor}
\begin{proof}
     Follows immediately from noting $\mathds{D}_t(Q) \leqq 2\max_{q_i \in Q} \widetilde{h}_t(q_i)$.
\end{proof}
\vspace{3pt}
\begin{cor}[Corollary To Proposition~\ref{propconvrate}]\label{corVsRcTg}Denoting an upper bound on the time required for all agents to converge to the target as $T_{conv}$, we have:
\begin{subequations}
 \cgather[2pt]{
 v_sT_{conv} \simeq Const. \textrm{ for constant $R_c$}\\
 R_c/ T_{conv}\simeq Const. \textrm{ for constant $v_s$}
 }
\end{subequations}
\end{cor}
\begin{proof}
     Replace the requirement of every agent converging to the target by one that requires almost all of them reaching the target, in the sense that the swarm diameter $\mathds{D}_{T_{conv}} \simeq f$, where $0<f\ll 1$. Using Corollary~\ref{corconvrate}, we have:
     \cgather[2pt]{
v_sT_{conv}/R_c \leqq \ln (2/f)     
     }
     Since $T_{conv}$ is an upper bound on the convergence time, we can replace the inequality:
     \cgather[2pt]{
v_sT_{conv}/R_c \simeq \ln (2/f) = Const.          
     }
     which implies the desired result.
\end{proof}

\begin{notn}
     We denote the set of unit vectors of velocity directions at time $t$ for the processes  $\mathscr{I}_{v_s}(t,\mathcal{P}(0))$ and $\mathscr{R}_{v_s}(t,\mathcal{P}(0))$ as $\partial\mathscr{I}_{v_s}(t,\mathcal{P}(0))$ and $\partial\mathscr{R}_{v_s}(t,\mathcal{P}(0))$ respectively.
\end{notn}
\vspace{3pt}

\begin{prop}[Asymptotic Deviation From Ideal Process]\label{propdev}
     For sufficiently small $v_s> 0$, we have:
     \cgather[2pt]{
     Prob\left (  \bigg \vert \bigg \vert 
     \partial \mathscr{R}_{v_s}(t,\mathcal{P}(0)) - \partial\mathscr{I}_{v_s}(t,\mathscr{R}_{v_s}(t,\mathcal{P}(0)))
     \bigg \vert \bigg \vert > 0
     \right ) \\= O\left (  \frac{R_c}{v_s} \left (  e^{(v_s/R_c) \mathds{T}_c}-1 \right ) e^{-(v_s/R_c) t}  \right )\label{eqawesome}
     }
     where $\mathds{T}_c$ is the convergence time of the frozen swarm.
\end{prop}
\begin{proof}
     We note that if the neighborhood maps and the failure probabilities do not change for the interval $[t-\mathds{T}_c,t]$, then the velocity vectors of the two processes will coincide. We assume that $v_s$  is small enough such that in the absence of topology updates, the velocity vectors for $\mathscr{R}_{v_s}(t,\mathcal{P}(0))$ coincide with that of  $\mathscr{I}_{v_s}(t,\mathscr{R}_{v_s}(t,\mathcal{P}(0)))$.
     Denoting the probability of  topology update at time $t$ as $\mathcal{T}(t)$, we note:
     \cgather[2pt]{
     Prob\left (  \bigg \vert \bigg \vert 
     \partial \mathscr{R}_{v_s}(t,\mathcal{P}(0)) - \partial\mathscr{I}_{v_s}(t,\mathscr{R}_{v_s}(t,\mathcal{P}(0)))
     \bigg \vert \bigg \vert > 0 \right )= \int_{t-\mathds{T}_c}^t \mathcal{T}(t')\mathrm{d}t' \label{eqint0}
     }
     $\mathcal{T}(t)$ however is bounded above by the fraction of agents $\psi(t)$ at time $t$ that have an inter-agent distance $\simeq R_c$, which is a necessary condition for such agents to affect a change in their neighborhood map.
     In particular, we have 
     \cgather[2pt]{
     \mathcal{T}(t) = O(\psi(t))
     }
     Since the swarm diameter $\mathds{D}_t(Q)$ is dominated by an exponentially decreasing function (See Corollary~\ref{corconvrate} to Proposition~\ref{propconvrate}), and inter-agent distance is bounded above by  $\mathds{D}_t(Q)$, we have:
     \cgather[2pt]{
     \psi(t) = O(e^{-(v_s/R_c) t})
     }
     The result then follows by standard algebra from Eq.~\eqref{eqint0}.
\end{proof}

Proposition~\ref{propdev} shows that the implementable process $\mathscr{R}_{v_s}(t,\mathcal{P}(0))$ starts coinciding, at least in probability, to the ideal update process for small agent velocities. Note, that as $v_s \rightarrow 0^+$ in Eq.~\eqref{eqawesome}, we have, as expected:
\cgather[2pt]{
\lim_{v_s\rightarrow 0} Prob\left (  \bigg \vert \bigg \vert 
     \partial \mathscr{R}_{v_s}(t,\mathcal{P}(0)) - \partial\mathscr{I}_{v_s}(t,\mathscr{R}_{v_s}(t,\mathcal{P}(0)))
     \bigg \vert \bigg \vert > 0
     \right )\notag \\= O\left ( R_c\lim_{v_s \rightarrow 0} \frac{ e^{(v_s/R_c) \mathds{T}_c}-1}{v_s} \right ) = O(\mathds{T}_c)
}
which reflects the fact that while $\mathscr{R}_{v_s}(t,\mathcal{P}(0))$ takes $O(\mathds{T}_c)$ time to converge, we assumed that $\mathscr{I}_{v_s}(t,\mathscr{R}_{v_s}(t,\mathcal{P}(0)))$ executes instantaneous route optimization. Similarly:
\cgather[2pt]{
\lim_{R_c\rightarrow \infty} Prob\left (  \bigg \vert \bigg \vert 
     \partial \mathscr{R}_{v_s}(t,\mathcal{P}(0)) - \partial\mathscr{I}_{v_s}(t,\mathscr{R}_{v_s}(t,\mathcal{P}(0)))
     \bigg \vert \bigg \vert > 0
     \right )  = O(\mathds{T}_c)
}
which indicates that in the case where no topology changes occur due to the fact that all agents are neighbors of each other, the ideal process is faster for the same reason. If there is no communication, then no  optimization is possible for either process:
\cgather[2pt]{
\lim_{R_c\rightarrow 0} Prob\left (  \bigg \vert \bigg \vert 
     \partial \mathscr{R}_{v_s}(t,\mathcal{P}(0)) - \partial\mathscr{I}_{v_s}(t,\mathscr{R}_{v_s}(t,\mathcal{P}(0)))
     \bigg \vert \bigg \vert > 0
     \right ) = 0
}

\section{Simulation Results}\label{sec6}
\subsection{Complexity of Optimization In The Frozen Swarm}
\begin{figure}[t]
\centering
\subfloat[]{\includegraphics[width=2.1in]{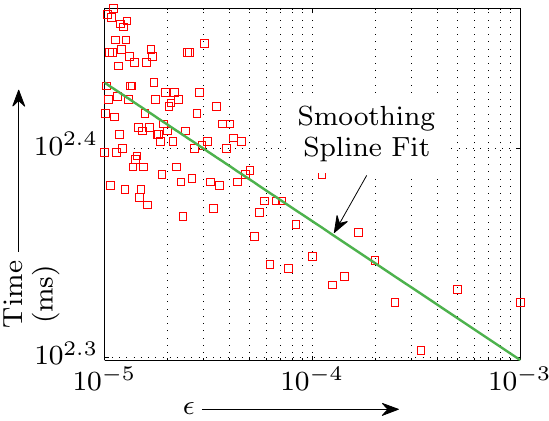}\label{figcomplexb}}
\subfloat[]{\includegraphics[width=2in]{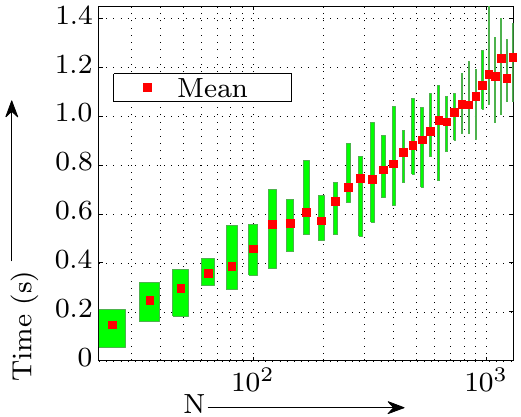}\label{figcomplexa}}\\
\subfloat[]{\includegraphics[width=2.1in]{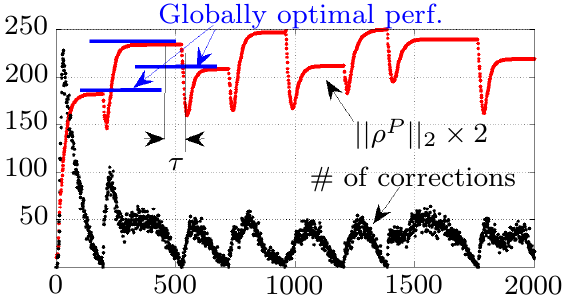}\label{figsim1a}}
\subfloat[]{\includegraphics[width=2.1in]{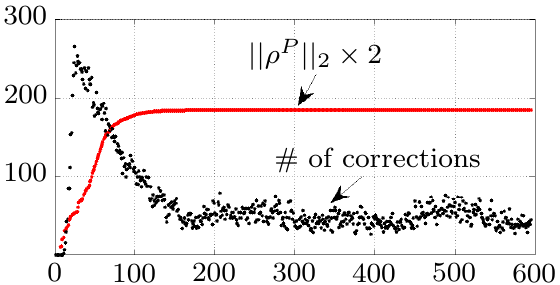}\label{figsim1b}}
\caption{ Convergence complexity for distributed  optimization of the frozen network: (a) illustrates  dependence  on network size. (b) captures the $O(1/\epsilon)$ dependence. Convergence dynamics: (c) rapid convergence to large random target movements  (d)robust response to  large zero-mean variations in the failure probabilities}\label{figsim1}
\end{figure}

Extensive simulations on NS2 network simulator are used to investigate how  convergence times scale as a function of the network size (Figure~\ref{figsim1}). $10^2$ random topologies were considered for each $N$ (increased from 25 to 1600), and the mean times along with the max-min bars are plotted in Figure~\ref{figcomplexa}.
Note that the abscissa is on a logarithmic scale, and the near linear nature of the plot indicates a logarithmic dependence of the convergence on network size, implying that the bound computed in Proposition~\ref{propcomplex} is possibly not tight. The dependence on $\epsilon$ shown in Figure~\ref{figcomplexb} (for $N=10^3$) is hyperbolic, as expected, leading to a near linear dependence after a smoothing spline fit on a log-log scale. Convergence times are estimated from NS2 output (using  802.11 standard).

Theoretical convergence results are illustrated in Figure~\ref{figsim1}(c-d),  generated on a $10^4$ frozen agent network. Figure~\ref{figsim1a} illustrates the variation of the number of route updates ($\#$ of forwarding decision corrections) and the norm of the performance vector $\rho^P$ (scaled up by a multiplicative factor of 2) when the target is moved around randomly at a slower time scale. Since $\rho^P$ is the vector of end-to-end success probabilities (See Definition~\ref{defperf}), its norm captures the degree of expected  throughput across the network. Note that target changes induce self-organizing corrections, which rapidly die down, with the performance converging close to the global optimal ($\epsilon=0.001$ was assumed in all the simulations). The failure probabilities are chosen randomly, and, on the average, held  constant  in the course of simulation illustrated in Figure~\ref{figsim1a} (zero mean Gaussian noise is added to illustrate robustness). Note that the seemingly large fluctuations in the performance norm is unavoidable; the interval $\tau$ is the what it approximately takes for information to percolate through the network, and hence  this much time is necessary at a minimum for decentralized route convergence. Figure~\ref{figsim1b} illustrates the effect of large zero-mean stochastic variations in the failure probabilities. Each agent estimates the failure probabilities from simple windowed average of the link-specific packet failures. We note that large sustained  fluctuations result in a sustained corrections in the forwarding decisions (which no longer goes to zero). However, the norm of the performance vector converges and holds steady. This clearly illustrates that the information percolation strategy induces a low-pass filter eliminating high-frequency  fluctuations. A small number of route fluctuations always occur (note the non-zero number of corrections), but this  does not induce significant performance variations.
\subsection{Simulation Studies On Mobile Swarms}
\begin{figure}[t]
\centering
     \subfloat[]{\includegraphics[width=2.3in]{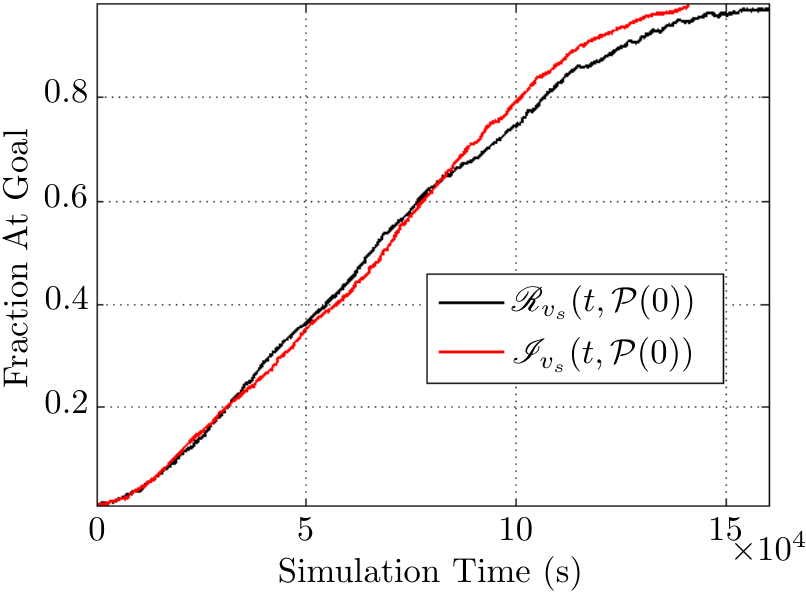}\label{figmobsima}}\hspace{10pt}
     \subfloat[]{\includegraphics[width=2.3in]{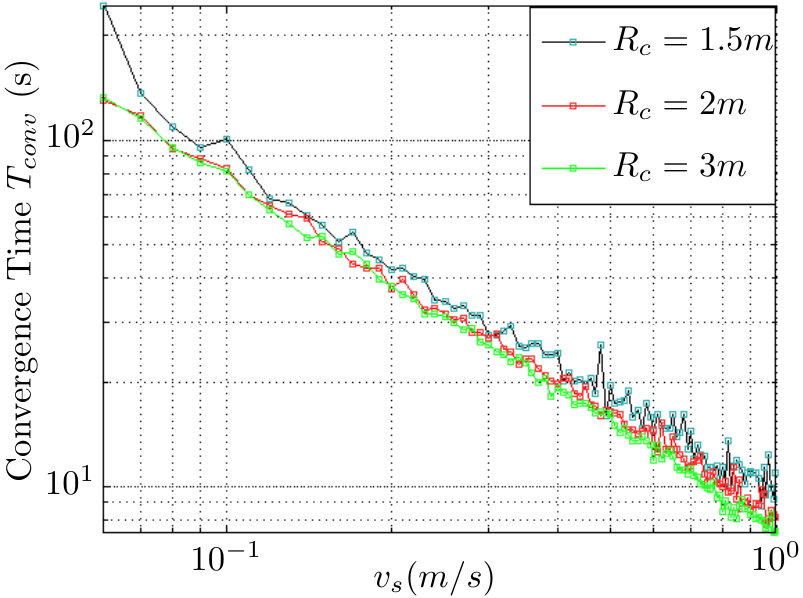}\label{figmobsimb}}\\
     \subfloat[]{\includegraphics[width=2.28in]{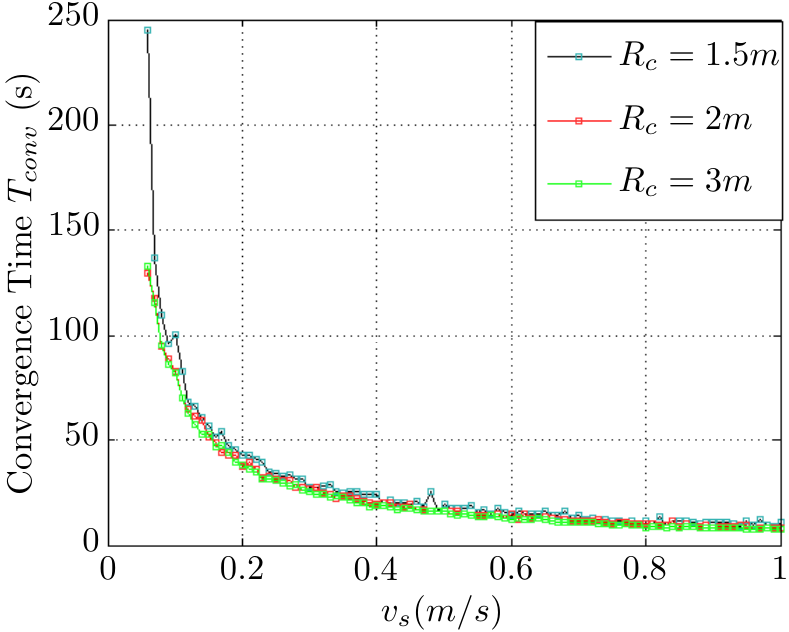}\label{figmobsimc}}\hspace{10pt}
     \subfloat[]{\includegraphics[width=2.25in]{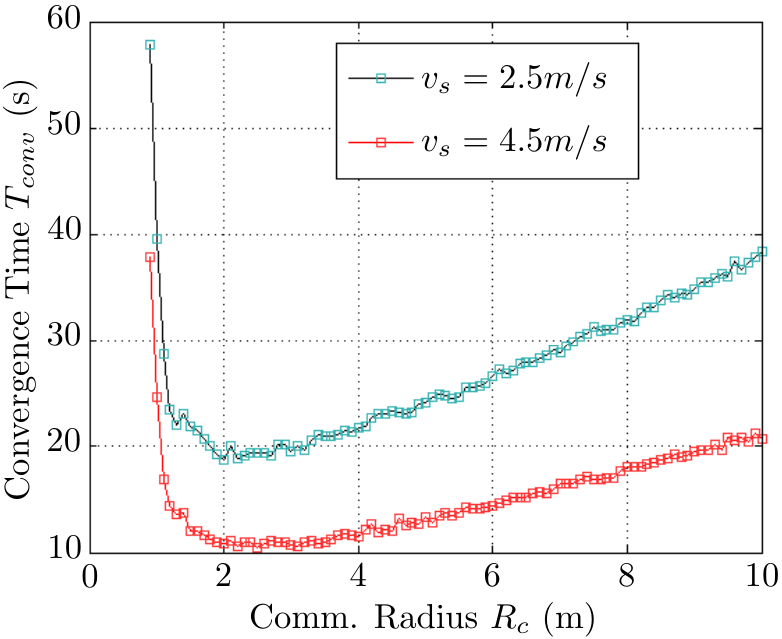}\label{figmobsimd}}
\caption{Simulation results validating theoretical development in Section~\ref{secMob}: (a)$\mathscr{I}_{v_s}(t,\mathcal{P}(0))$ and $\mathscr{R}_{v_s}(t,\mathcal{P}(0))$ matches closely w.r.t. the the fraction reaching target, (b) variation of $T_{conv}$ with velocity $v_s$ at constant communication radius $R_c$ on a log-log scale showing the predicted hyperbolic relationship (Corollary~\ref{corVsRcTg}), (c) same data on linear scale, (d) variation of $T_{conv}$ with $R_c$ at constant $v_s$ showing the predicted linear relationship. At low values of $R_c$ swarm begins to get disconnected resulting in rapid increase in $T_{conv}$ }\label{figmobsim}
     
\end{figure}
\begin{figure}[t]
     \includegraphics[width=5in]{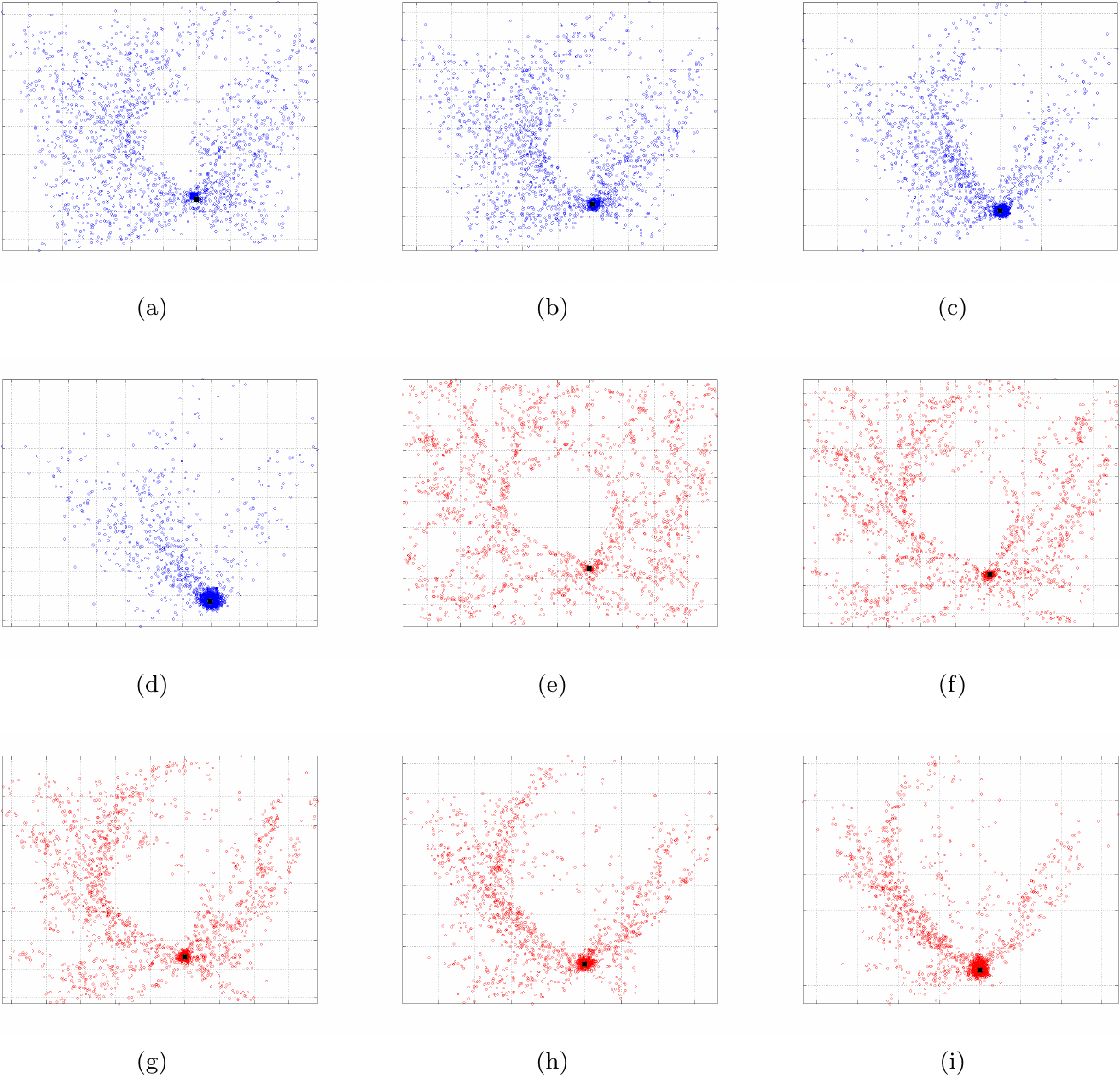}
     \caption{Effect of communication radius  $R_c$ on swarm trajectories: Plates (a)-(d) illustrate the case for high $R_c=3m$, whereas plates (e)-(i) illustrate the case for $R_c=0.1m$. Note the case for low $R_c$ develops high-traffic paths in the simulation, whereas for high $R_c$, the trajectories are more amorphous}\label{figsimcase5}
\end{figure}

First, we need to validate the theoretical development presented in Section~\ref{secMob}.
For that purpose, we consider a swarm of $10^4$ agents, with a single target. The swarms are initially distributed uniformly over a $100m \times 100m$ region. The failure probabilities are assumed to be linearly decreasing functions of the inter-agent distances. They also are given a random spatial dependence. The system is simulated in accordance to the described algorithms, at various values of the communication radius $R_c$, and the agent velocity $v_s$. Convergence time $T_{conv}$ is the time required for nearly all agents ($> 99.9\%$) to reach the target. The results are shown in Figure~\ref{figmobsim}.
Note that Figure~\ref{figmobsima} validates Proposition~\ref{propdev} in that the ideal process $\mathscr{I}_{v_s}(t,\mathcal{P}(0))$ closely matches the implementable process    
$\mathscr{R}_{v_s}(t,\mathcal{P}(0))$ with respect to the fraction of agents reaching the goal as a function of the simulation time. For our simulated system the ideal process $\mathscr{I}_{v_s}(t,\mathcal{P}(0))$ was obtained by freezing the swarm for $1000$ simulation ticks after every tick that caused any position updates. Since the frozen swarm is guaranteed to converge to the optimal routes (as shown in Section~\ref{sec4}), this procedure ensures that the achieved position updates closely approximate the ideal process. In logging simulation time, we ignored the time spent in the frozen optimization. Figure~\ref{figmobsimb}-\ref{figmobsimd} validates Corollary~\ref{corVsRcTg}, by showing that for constant communication radius $R_c$, the convergence time $T_{conv}$ does indeed vary hyperbolically with the $v_s$(Figure~\ref{figmobsimb},\ref{figmobsimc}), and for constant $v_s$, it increases linearly with the communication radius $R_c$ (Figure~\ref{figmobsimd}). As the communication radius is decreased, we eventually encounter a point where the swarm begins to get disconnected, which is reflected in the rapid increase of $T_{conv}$ at low values of $R_c$. If we identify $R_c$ with the amount of energy spent in communication, we note that there is an optimum value at which the $T_{conv}$ is minimized. The high-traffic paths generated in the swarm are very different for different communication radius. This is illustrated in Figure~\ref{figsimcase5} (with $10^4$ agents), where we note that for small $R_c$ we see the development of distinct paths.

\begin{figure}[t]
     \includegraphics[width=5in]{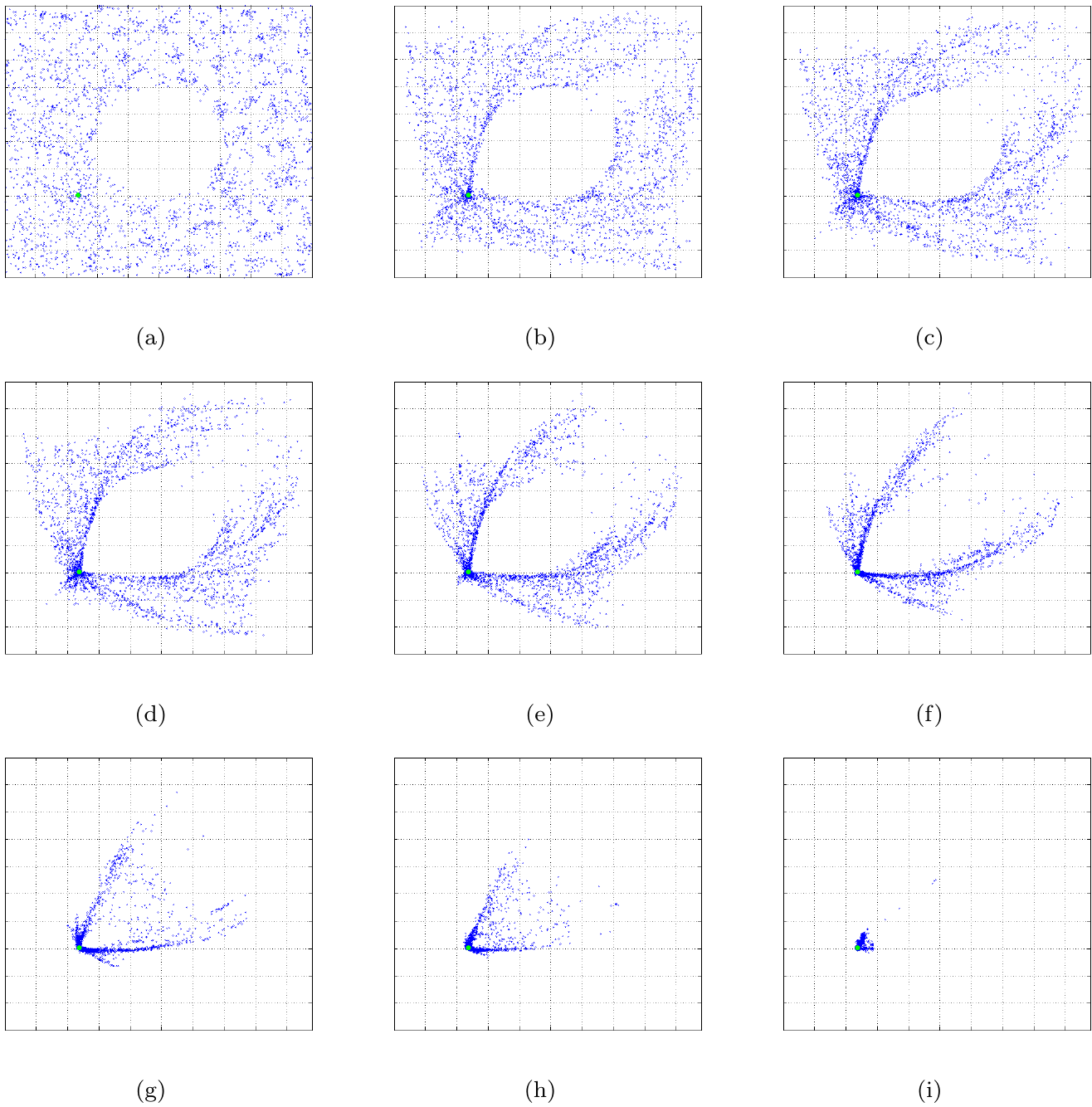}
     \caption{Illustrating effect of void in the initial agent distribution: the trajectories tend to avoid the void since every agent is following a neighbor leading to distinct paths}\label{figsimcase1}
\end{figure}

\begin{figure}[t]
     \includegraphics[width=5in]{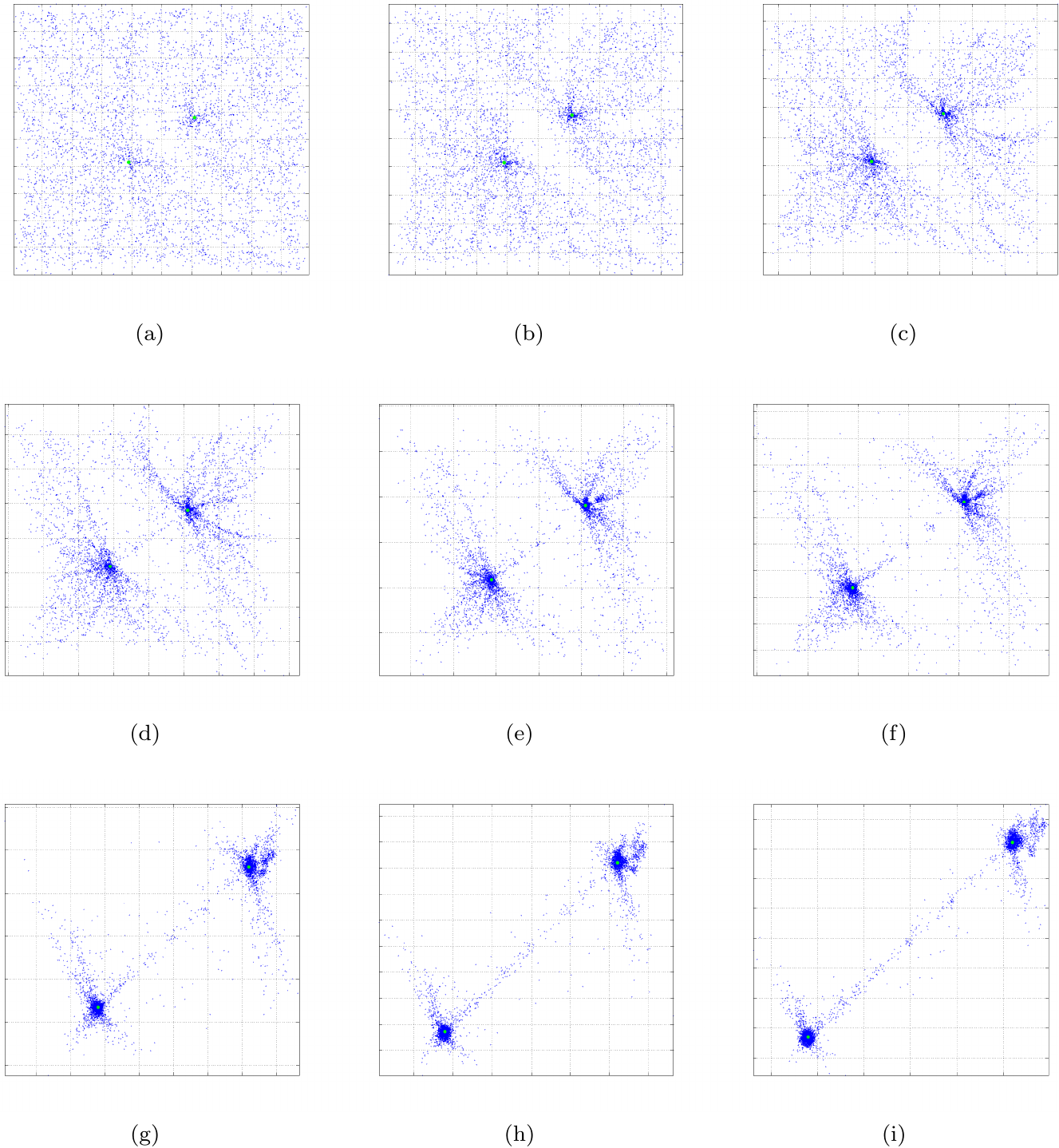}
     \caption{Multiple point targets: Note that the swarm automatically splits to move towards chosen targets; such choices are not made a priori but naturally emerge from the execution  }\label{figsimcase3}
\end{figure}

\begin{figure}[t]
     \includegraphics[width=5in]{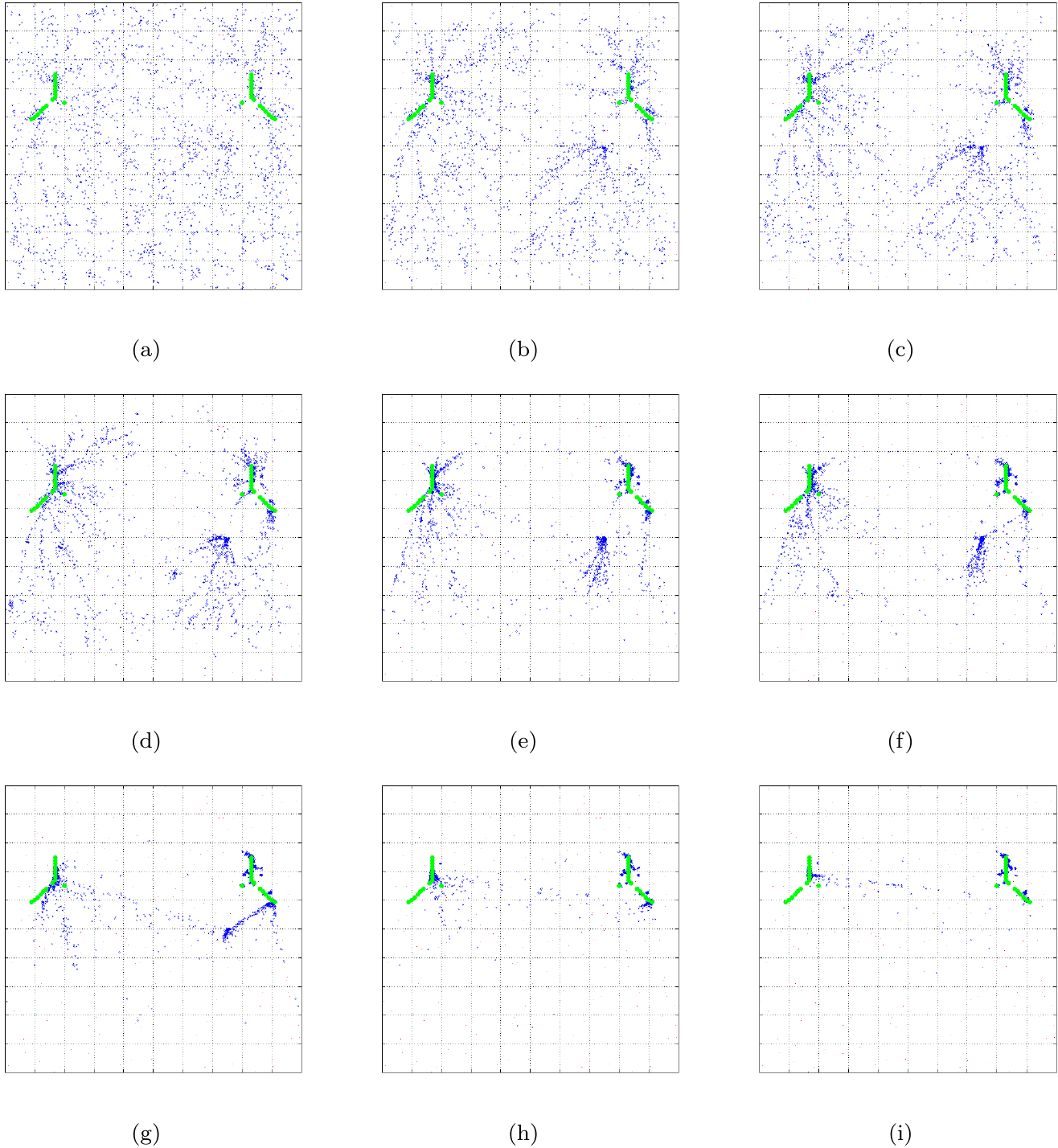}
     \caption{Illustrating multiple extended target regions: Note that the swarm automatically splits to move towards chosen targets (in green); as before such choices are not made a priori but naturally emerge from the execution  }\label{figsimcase2}
\end{figure}

\begin{figure}[t]
     \includegraphics[width=5in]{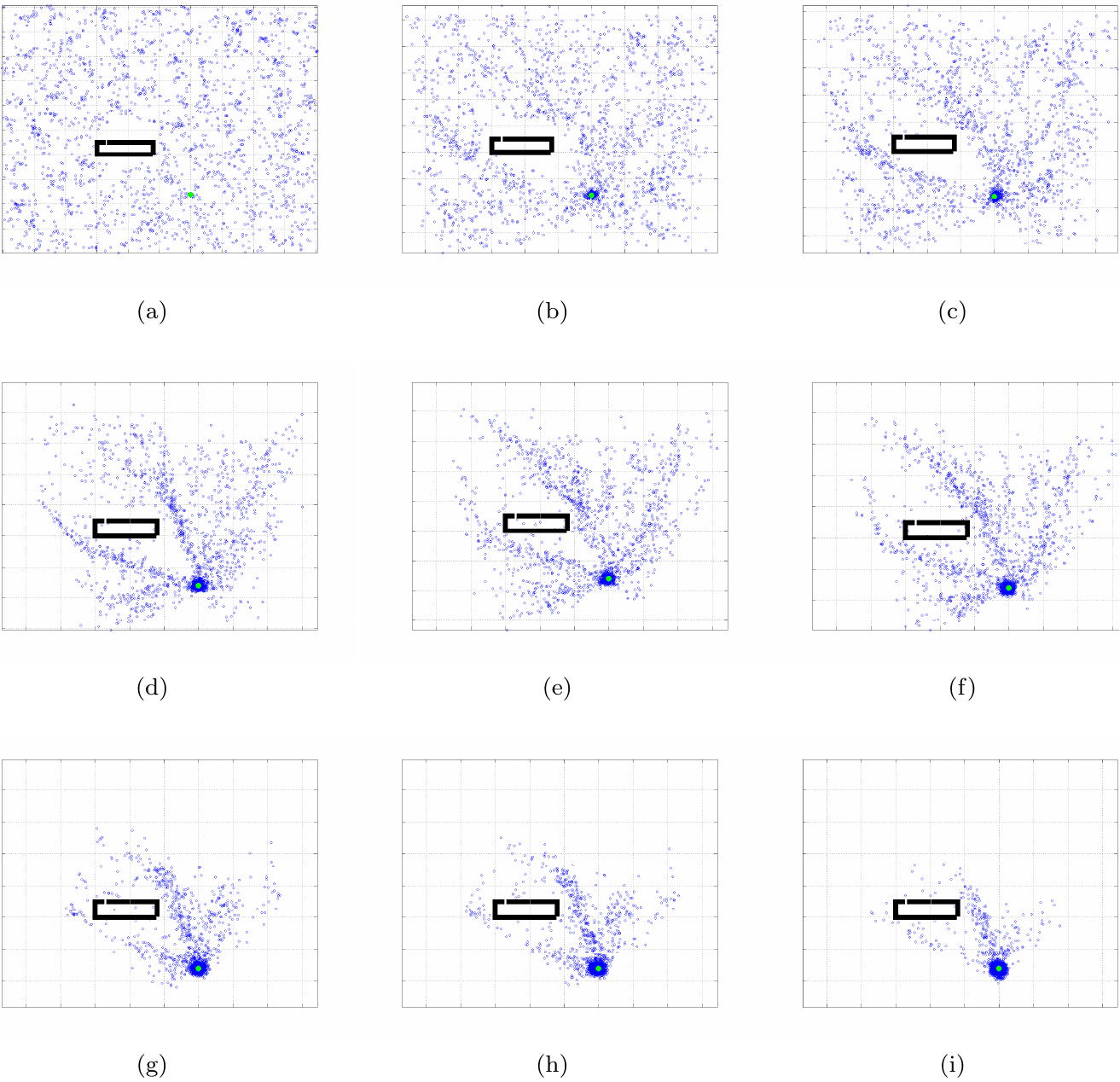}
     \caption{Presence of obstacles with obstacle positions are locally sensed and are not known a priori or globally. Note how the high traffic paths go around}\label{figsimcase4}
\end{figure}

\begin{figure}[t]
     \includegraphics[width=5in]{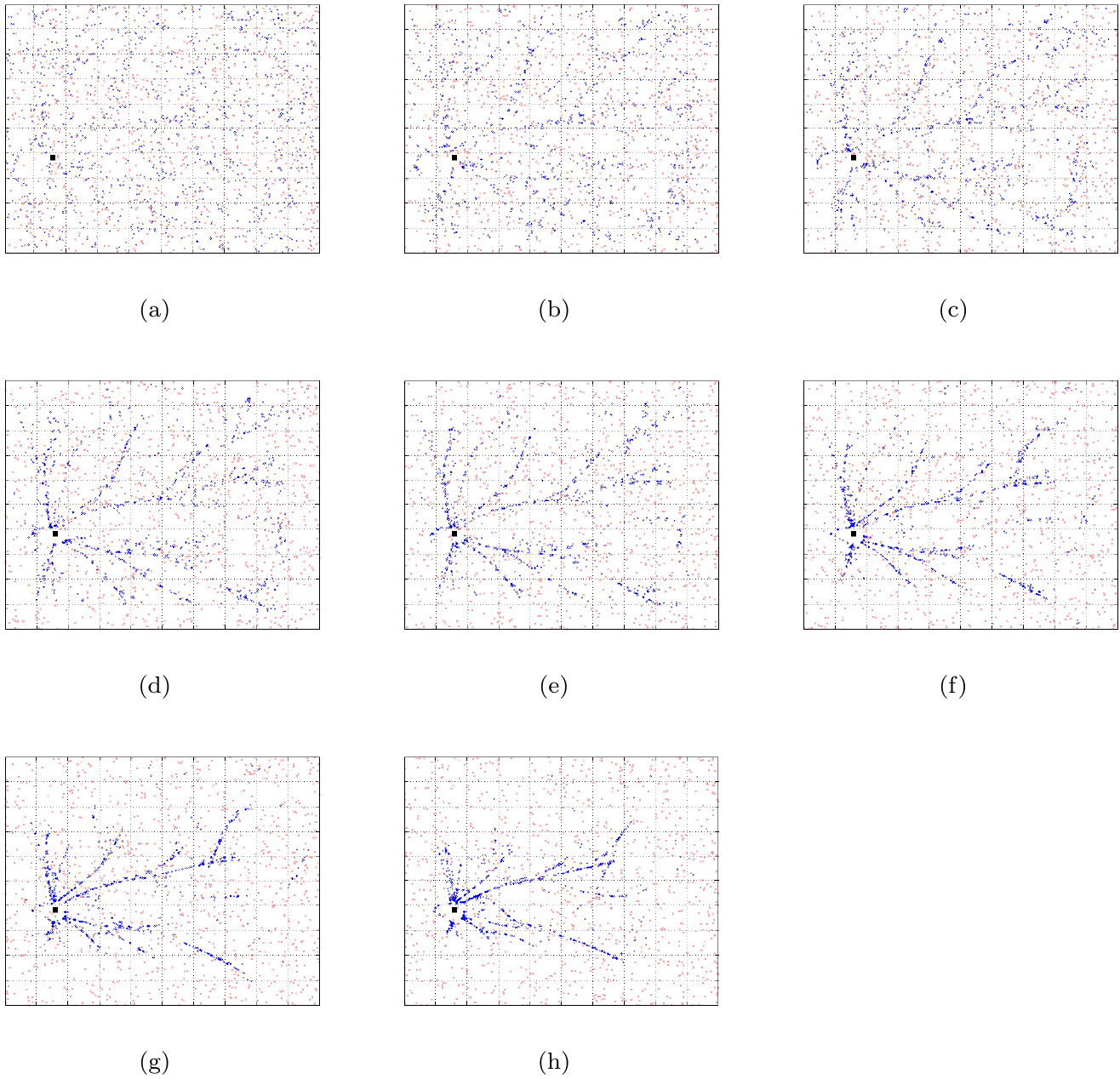}
     \caption{Presence of randomly moving agents (in red) which are not going to the target, but are willing to share information. Note how the blue agents organize to form paths }\label{figsimcase6}
\end{figure}
\subsection{Simulation Case Studies}
We present simulation results for a series of different scenarios. All simulations are done with a minimum of $10^4$ agents, and the communication radius $R_c$ and the swarm velocity $v_s$ is kept fixed at $0.1m$, and $2.5m/s$ unless otherwise stated. The initial agent distribution is uniform, as before, over a $100m \times 100m$ plane (with the exception of intentional voids).
\begin{enumerate}
     \item Figure~\ref{figsimcase1} illustrates the situation where there is a void in the initial agent distribution. We note how the high-traffic paths go around the void, resulting from the neighbor-following mechanism. This is intuitively correct, as the small communication radius implies that no information is available on the failure probabilities inside the void. As the simulation progresses, the two paths going around the void closes together, illustrating the fact as agents on the edge of the void incur inwards due to noisy position updates, the area with no agents diminishes.
     \item Figure~\ref{figsimcase3} illustrates the scenario where we have two point targets. Note how the swarm splits automatically, and moves towards a particular target. This splitting decision is not taken a priori, but naturally emerges from the execution.
     \item Figure~\ref{figsimcase4} illustrates a scenario with multiple extended targets. Note how the swarm trajectories form temporary accumulation points (can be seen towards the bottom of the target on the right), indicating the locations from where going towards both targets are equally advantageous. Random choice of best neighbors is sufficient, without any addition effort, to ensure that such accumulations are only temporary. 
     \item Figure~\ref{figsimcase5} illustrates the scenario with an obstacle, whose position needs to be locally sensed, and is not known a priori or globally.
     \item Figure~\ref{figsimcase6} illustrates the case with agents that are not interested in going to target, but are willing to share information $i.e.$ carry out the node-based computations, and relate the measure values to the neighbors. The difficulty is that such agents are not necessarily moving towards the goal, so following such an agent may prove detrimental. However, as these agents move in a direction which is not improving the measure gradient, their nodal updates begin diminishing their current measures. The simulation shows that a sparse number of agents (in blue) which intend to reach the target can accomplish this in the presence of the former type (in red). Note also that in this simulation the number of blue agents was chosen to be insufficient to form a connected network. Thus the information percolation is shown to be sufficient for the development of distinct paths to target.
\end{enumerate}

\section{Future Work}\label{sec7}
%
%
Future work will proceed in the following directions:
\begin{enumerate}
 \item Design explicit strategies for energy and congestion awareness within the proposed framework.  Note that each agent can regulate incoming traffic by deliberately reporting  lower values of its current self-measure to its neighbors: 
\begin{gather}\label{eqfactor}
\textrm{\small \sffamily \BRed Reported} \longrightarrow r^{[k]}_{\theta}\big \vert_i     = \zeta(q_i,k)\nu^{[k]}_{\theta}\big \vert_i \leftarrow \textrm{\small \sffamily \Dgreen Computed}
\end{gather}
where $\forall q_i \in Q, k \in[0,\infty), \zeta(q_i,k) \in [0,1]$ is a multiplicative factor which is modulated to have  decreasing values as  local congestion increases. Such modulation  forces automatic self-organization to compute alternate routes that tend to avoid the particular agent. The dynamics of such context-aware modulation  may  be non-trivial; while for slowly varying $\zeta(q_i,k)$, the convergence results presented here is expected to hold true, rapid fluctuations in  $\zeta(q_i,k)$  may be problematic.

\item We assumed that the link-specific failure probabilities are  estimated at the agents. Grossly incorrect estimations will translate to incorrect  routing decisions, and decentralized strategies for 
robust identification of these parameters need to be investigated at a greater depth. More specifically, the proposed algorithm needs to be augmented with learning schemes (possibly reinforcement learning based approaches) that estimate such failure probabilities.
\item Explicit design of implementation details such as packet headers, agent data structures and pertinent neighbor-neighbor communication protocols.
\item Hardware validation  on real systems.
\end{enumerate}

}
\bibliographystyle{IEEEtran}
\bibliography{BibLib,ref}

\begin{thebibliography}{10}
\providecommand{\url}[1]{#1}
\csname url@samestyle\endcsname
\providecommand{\newblock}{\relax}
\providecommand{\bibinfo}[2]{#2}
\providecommand{\BIBentrySTDinterwordspacing}{\spaceskip=0pt\relax}
\providecommand{\BIBentryALTinterwordstretchfactor}{4}
\providecommand{\BIBentryALTinterwordspacing}{\spaceskip=\fontdimen2\font plus
\BIBentryALTinterwordstretchfactor\fontdimen3\font minus
  \fontdimen4\font\relax}
\providecommand{\BIBforeignlanguage}[2]{{%
\expandafter\ifx\csname l@#1\endcsname\relax
\typeout{** WARNING: IEEEtran.bst: No hyphenation pattern has been}%
\typeout{** loaded for the language `#1'. Using the pattern for}%
\typeout{** the default language instead.}%
\else
\language=\csname l@#1\endcsname
\fi
#2}}
\providecommand{\BIBdecl}{\relax}
\BIBdecl

\bibitem{DG89}
J.~Deneubourg and S.~Goss, ``Collective patterns and decision-making,''
  \emph{Ethology, Ecology, and Evolution}, vol.~1, 1989.

\bibitem{Gage1992}
D.~Gage, ``Command and control for many-robot systems,'' \emph{In Unmanned
  Systems Magazine}, vol.~10, 1992.

\bibitem{Holl99}
\BIBentryALTinterwordspacing
O.~Holland and C.~Melhuish, ``Stigmergy, self-organization, and sorting in
  collective robotics,'' \emph{Artif. Life}, vol.~5, pp. 173--202, April 1999.
  [Online]. Available:
  \url{http://portal.acm.org/citation.cfm?id=338945.338956}
\BIBentrySTDinterwordspacing

\bibitem{Unsal1994}
C.~Unsal and J.~Bay, ``Spatial self-organization in large populations of mobile
  robots,'' in \emph{IEEE Int. Symp. on Intelligent Control}, 1994, pp.
  249--254.

\bibitem{Gage93}
D.~W. Gage, ``How to communicate with zillions of robots,'' in \emph{In
  Proceedings of SPIE Mobile Robots VIII}, 1993, pp. 250--257.

\bibitem{LB92}
M.~Lewis and G.~Bekey, ``The behavioral self-organization of nanorobots using
  local rules,'' in \emph{In Proc. 1992 IEEE/RSJ Int. Conf. Intelligent Robots
  and Systems}, 1992.

\bibitem{Payton01}
D.~Payton, R.~Estkowski, and M.~Howard, ``Compound behaviors in pheromone
  robotics,'' \emph{Robotics and Autonomous Systems}, vol.~44, no. 3-4, pp.
  229--240, 2001.

\bibitem{CR07}
I.~Chattopadhyay and A.~Ray, ``Language-measure-theoretic optimal control of
  probabilistic finite-state systems,'' \emph{Int. J. Control}, 2007.

\bibitem{CR08}
------, ``Structural transformations of probabilistic finite state machines,''
  \emph{International Journal of Control}, vol.~81, no.~5, pp. 820--835, May
  2008.

\bibitem{STMRK09}
T.~Schmickl, R.~Thenius, C.~Moeslinger, G.~Radspieler, S.~Kernbach,
  M.~Szymanski, , and K.~Crailsheim, ``Get in touch: cooperative decision
  making based on robot-to-robot collisions,'' \emph{Autonomous Agents and
  Multi-Agent Systems}, vol.~18, no.~1, pp. 133--155, 2009.

\bibitem{NCD08}
S.~Nouyan, A.~Campo, and M.~Dorigo, ``Path formation in a robot swarm -
  self-organized strategies to find your way home,'' \emph{Swarm Intelligence},
  vol.~2, no.~1, pp. 1--23, 2008.

\bibitem{HW09}
B.~Holldobler and E.~O. Wilson, \emph{The superorganism: the beauty, elegance,
  and strangeness of insect societies}.\hskip 1em plus 0.5em minus 0.4em\relax
  W.W. Norton \& Company, 2009.

\bibitem{Svennebring04}
J.~Svennebring and S.~Koenig, ``Building terrain-covering ant robots: A
  feasibility study,'' \emph{Autonomous Robots}, vol.~16, no.~3, pp. 313--332,
  2004.

\bibitem{WL99}
I.~Wagner, M.~Lindenbaum, and A.~Bruckstein, ``Distributed covering by ant-
  robots using evaporating traces,'' \emph{IEEE Transactions on Robotics and
  Automation}, vol.~15, 1999.

\bibitem{Couzin2005}
\BIBentryALTinterwordspacing
I.~D. Couzin, J.~Krause, N.~R. Franks, and S.~A. Levin, ``Effective leadership
  and decision-making in animal groups on the move,'' \emph{Nature}, vol. 433,
  no. 7025, pp. 513--516, Feb 2005. [Online]. Available:
  \url{http://dx.doi.org/10.1038/nature03236}
\BIBentrySTDinterwordspacing

\bibitem{Couzin2007}
\BIBentryALTinterwordspacing
I.~Couzin, ``Collective minds,'' \emph{Nature}, vol. 445, no. 7129, pp.
  715--715, Feb 2007. [Online]. Available:
  \url{http://dx.doi.org/10.1038/445715a}
\BIBentrySTDinterwordspacing

\bibitem{Nolfi00}
S.~Nolfi and D.~Floreano, \emph{Evolutionary Robotics: The
  Biology,Intelligence,and Technology}.\hskip 1em plus 0.5em minus 0.4em\relax
  Cambridge, MA, USA: MIT Press, 2000.

\bibitem{Bn00}
D.~S. Bernstein, R.~Givan, N.~Immerman, and S.~Zilberstein, ``The complexity of
  decentralized control of markov decision processes,'' in \emph{Proceedings of
  the Sixteenth Conference on Uncertainty in Artificial Intelligence
  (UAI-2000)}, 2000, pp. 32--37.

\bibitem{BGIZ02}
------, ``The complexity of decentralized control of markov decision
  processes,'' \emph{Math. Oper. Res.}, vol.~27, no.~4, pp. 819--840, 2002.

\bibitem{LGM01}
C.~Lusena, J.~Goldsmith, and M.~Mundhenk, ``Nonapproximability results for
  partially observable markov decision processes,'' \emph{Journal of Artificial
  Intelligence Research}, vol.~14, p. 2001, 2001.

\bibitem{KGZ07}
M.~Kumar, D.~Garg, and R.~Zachery, ``Multiple mobile agents control via
  artificial potential functions and random motion,'' in \emph{Proceedings of
  the ASME International Mechanical Engineering Congress and Exposition}.\hskip
  1em plus 0.5em minus 0.4em\relax Seattle, WA: ASME, November 2007.

\bibitem{SHK08}
S.~Sarkar, E.~Halland, and M.~Kumar, ``Mobile robot path planning using support
  vector machines,'' in \emph{ASME Dynamic Systems and Control
  Conference}.\hskip 1em plus 0.5em minus 0.4em\relax Ann Arbor, Michigan:
  ASME, 2008.

\bibitem{BK91-1}
J.~Borenstein and Y.~Koren, ``Potential field methods and their inherent
  limitations for mobile robot navigation,'' in \emph{Proceedings of the 1991
  IEEE International Conference on Robotics and Automation}, 1991, pp.
  1398--1404.

\bibitem{Volpe1990}
R.~Volpe and P.~Khosla, ``Manipulator control with superquadric artificial
  potential functions: theory and experiments,'' \emph{IEEE Transactions on
  Systems, Man, and Cybernetics,}, vol.~20, pp. 1423--1436, 1990.

\bibitem{Connolly1997}
C.~Connolly, ``Harmonic functions and collision probabilities,'' \emph{The
  International Journal of Robotics Research,}, vol.~16, pp. 497--507, 1997.

\bibitem{Vascak2007}
J.~Vascak, ``Navigation of mobile robots using potential fields and
  computational intelligence means,'' \emph{Acta Polytechnica Hungarica,},
  vol.~4, pp. 63--74, 2007.

\bibitem{KK93}
J.~Kim and P.~K. Khosla, ``Real-time obstacle avoidance using harmonic
  potential functions,'' \emph{IEEE Transactions on Robotics and Automation},
  vol.~8, no.~3, pp. 338--349, 1993.

\bibitem{S93}
K.~Sato, ``Deadlock-free motion planning using the laplace potential field,''
  \emph{Advanced Robotics}, vol.~5, no.~3, pp. 449--461, 1993.

\bibitem{AKH08}
Akishita, S.~Kawamura, and K.~Hayashi, ``New navigation function utilizing
  hydrodynamic potential for mobile robot,'' in \emph{IEEE Int. Workshop on
  Intelligent Motion Control}, 1990, pp. 413--417.

\bibitem{WM03}
S.~Waydo and R.~M. Murray, ``Vehicle motion planning using stream functions,''
  in \emph{Proc. IEEE Int. Conf. Robotics and Automation}, 2003, pp.
  2484--2491.

\bibitem{Shar1993}
F.~Sharifi and D.~Vinke, ``Integration of the artificial potential field
  approach with simulated annealing for robot path planning,'' 1993, pp.
  536--541.

\bibitem{Barraquand1992}
J.~Barraquand, B.~Langlois, and J.-C. Latombe, ``Numerical potential field
  techniques for robot path planning,'' \emph{IEEE Transactions on Systems,
  Man, and Cybernetics,}, vol.~22, pp. 224--241, 1992.

\bibitem{Barraquand1991}
J.~Barraquand and J.~Latombe, ``Robot motion planning: a distributed
  representation approach,'' \emph{The International Journal of Robotics
  Research,}, vol.~10, pp. 628--6449, 1991.

\bibitem{Park2003}
M.~Park and M.~Lee, ``A new technique to escape local minimum in artificial
  potential field based path planning,'' \emph{KSME International Journal,},
  vol.~17, pp. 1876--1885, 2003.

\bibitem{Li2000}
C.~Li, H.~Marcelo, H.~Krishnan, and L.~Yong, ``Virtual obstacle concept for
  local-minimum-recovery in potential-field based navigation,'' 2000, pp.
  983--988.

\bibitem{Bell2004}
G.~Bell and M.~Weir, ``Forward chaining for robot and agent navigation using
  potential fields,'' 2004, pp. 265--274.

\bibitem{Weir2006}
M.~Weir, A.~Buck, and J.~Lewis, ``A mind's eye approach to providing bug-like
  guarantees for adaptive obstacle navigation using dynamic potential fields,''
  \emph{Lecture Notes in Computer Science}, vol. 4095, pp. 239--250, 2006.

\bibitem{Borenstein1989}
J.~Borenstein and Y.~Koren, ``Real-time obstacle avoidance for fast mobile
  robots,'' \emph{IEEE Transactions on Systems, Man, and Cybernetics,},
  vol.~19, pp. 1179--1187, 1989.

\bibitem{Yun1997}
X.~Yun and K.~Tan, ``A wall-following method for escaping local minima in
  potential field based motion planning,'' in \emph{Proceedings of , 8th
  International Conference on Advanced Robotics}, 1997.

\bibitem{Mabrouk2008w}
M.~Mabrouk and C.~McInnes, ``An emergent wall following behaviour to escape
  local minima for swarms of agents,'' \emph{IAENG International Journal of
  Computer}, vol.~35, pp. 463--476, 2008.

\bibitem{HMU01}
J.~E. Hopcroft, R.~Motwani, and J.~D. Ullman, \emph{Introduction to Automata
  Theory, Languages, and Computation, 2nd ed.}\hskip 1em plus 0.5em minus
  0.4em\relax Addison-Wesley, 2001.

\bibitem{BR97}
R.~Bapat and T.~Raghavan, \emph{Nonnegative matrices and Applications}.\hskip
  1em plus 0.5em minus 0.4em\relax Cambridge University Press, 1997.

\bibitem{R88}
W.~Rudin, \emph{Real and Complex Analysis, 3rd ed.}\hskip 1em plus 0.5em minus
  0.4em\relax McGraw Hill, New York, 1988.

\bibitem{FS02}
E.~A. Feinberg and A.~Shwartz, Eds., \emph{Handbook of Markov Decision
  Processes: Methods and Algorithms}.\hskip 1em plus 0.5em minus 0.4em\relax
  Boston, MA: Kluwer, 2002.

\bibitem{Bertsekas1978}
D.~P. Bertsekas and S.~E. Shreve, \emph{Stochastic Optimal Control: The
  Discrete Time Case}.\hskip 1em plus 0.5em minus 0.4em\relax New York:
  Academic Press, 1978.

\bibitem{Bertsekas1987}
D.~P. Bertsekas, \emph{Dynamic Programming: Deterministic and Stochastic
  Models}.\hskip 1em plus 0.5em minus 0.4em\relax Englewood Cliffs, NJ:
  Prentice Hall, 1987.

\bibitem{RW87}
P.~J. Ramadge and W.~M. Wonham, ``Supervisory control of a class of discrete
  event processes,'' \emph{SIAM J. Control and Optimization}, vol.~25, no.~1,
  pp. 206--230, 1987.

\bibitem{Put90}
M.~L. Puterman, \emph{Markov decision processes in Handbooks in Operation
  Research and Management Science}, D.~P. Heyman and M.~J. Sobel, Eds.\hskip
  1em plus 0.5em minus 0.4em\relax Amsterdam: North–Holland, 1990.

\bibitem{GS75}
L.~G. Gubenko and E.~S. Statland, ``On controlled, discrete-time markov
  decision processes,'' \emph{Theory Probab. Math. Statist.}, vol.~7, p.
  47–61., 1975.

\bibitem{ABEFG93}
A.~Arapostathis, V.~S. Borkar, E.~Fernandez-Gaucherand, and M.~K. Ghosh,
  ``Discrete-time controlled markov processes with average cost criterion: A
  survey,'' \emph{SIAM J. Control and Optimization}, vol.~31, pp. 282--344,
  1993.

\bibitem{G31}
S.~Gerschgorin, ``\"{U}ber die abgrenzung der eigenwerte einer matrix,''
  \emph{Izv. Akad. Nauk. USSR Otd. Fiz.-Mat. Nauk}, vol.~7, pp. 749--754, 1931.

\bibitem{V04}
R.~S. Varga, \emph{Gerschgorin and His Circles}.\hskip 1em plus 0.5em minus
  0.4em\relax Germany: Springer, 2004.

\end{thebibliography}
\end{document}